\documentclass[twoside]{article}
\usepackage[accepted]{aistats2023}
\usepackage[round]{natbib}

\usepackage[english]{babel}

\usepackage[utf8]{inputenc}

\usepackage{amsmath,amssymb,amsfonts,mathrsfs}
\allowdisplaybreaks

\usepackage[amsmath,thmmarks]{ntheorem}

\usepackage{graphicx}

\usepackage{pdfpages}

\usepackage{varioref}

\usepackage{datetime}

\usepackage{mathtools}
\mathtoolsset{showonlyrefs}  

\usepackage[h]{esvect}

\usepackage{array}

\usepackage{listings}
\usepackage{microtype}

\usepackage{booktabs}

\usepackage{algorithm, algorithmic}

\usepackage[labelfont=rm]{subcaption}

\usepackage{tikz}
\usetikzlibrary{bayesnet}
\usepackage{footnote}
\usepackage{bbm}
\usepackage{dsfont}
\usepackage{soul}

\newtheorem{theorem}{Theorem}
\newtheorem{example}{Example}

\newtheorem{lemma}{Lemma}

\newtheorem{assumption}{Assumption}

\theoremstyle{nonumberplain}
\theorembodyfont{\normalfont}
\theoremsymbol{\ensuremath{\square}}
\newtheorem{proof}{Proof}

\makeatletter
\@addtoreset{theorem}{section}
\makeatother

\newcommand{\bszero}{\boldsymbol{0}}
\newcommand{\bsone}{\boldsymbol{1}}

\newcommand{\bse}{\boldsymbol{e}}

\newcommand{\bsg}{\boldsymbol{g}}
\newcommand{\bsp}{\boldsymbol{p}}
\newcommand{\bsw}{\boldsymbol{w}}

\newcommand{\bsz}{\boldsymbol{z}}

\newcommand{\bsphi}{\boldsymbol{\phi}}
\newcommand{\bspsi}{\boldsymbol{\psi}}

\newcommand{\EE}{\mathbb{E}}

\newcommand{\PP}{\mathbb{P}}

\newcommand{\RR}{\mathbb{R}}
\newcommand{\VV}{\mathbb{V}}

\newcommand{\cE}{\mathcal{E}}
\newcommand{\cH}{\mathcal{H}}

\newcommand{\cO}{\mathcal{O}}
\newcommand{\cT}{\mathcal{T}}
\newcommand{\cU}{\mathcal{U}}
\newcommand{\cW}{\mathcal{W}}
\newcommand{\cX}{\mathcal{X}}
\newcommand{\cY}{\mathcal{Y}}
\newcommand{\cZ}{\mathcal{Z}}

\newcommand{\obs}{\mathrm{obs}}
\newcommand{\base}{\mathrm{base}}
\newcommand{\rbox}{\mathrm{box}}

\newcommand{\textinf}{\mathrm{inf}}
\newcommand{\textsup}{\mathrm{sup}}
\newcommand{\ZSB}{\mathrm{ZSB}}
\newcommand{\CMC}{\mathrm{CMC}}
\newcommand{\KCMC}{\mathrm{KCMC}}
\newcommand{\KPCA}{\mathrm{KPCA}}
\newcommand{\IPW}{\mathrm{IPW}}

\newcommand{\rd}{\, \mathrm{d}}
\newcommand{\pto}{{\overset{p.}{\to}}}

\newcommand{\bbmone}{\mathds{1}}  

\newcommand{\indep}{\perp \!\!\! \perp}

\renewcommand{\epsilon}{\ensuremath\varepsilon}

\renewcommand{\phi}{\ensuremath{\varphi}}

\DeclareMathOperator*{\sumint}{
\mathchoice
  {\ooalign{$\displaystyle\sum$\cr\hidewidth$\displaystyle\int$\hidewidth\cr}}
  {\ooalign{\raisebox{.14\height}{\scalebox{.7}{$\textstyle\sum$}}\cr\hidewidth$\textstyle\int$\hidewidth\cr}}
  {\ooalign{\raisebox{.2\height}{\scalebox{.6}{$\scriptstyle\sum$}}\cr$\scriptstyle\int$\cr}}
  {\ooalign{\raisebox{.2\height}{\scalebox{.6}{$\scriptstyle\sum$}}\cr$\scriptstyle\int$\cr}}
}
\usepackage[colorinlistoftodos,bordercolor=orange,backgroundcolor=orange!20,linecolor=orange,textsize=scriptsize]{todonotes}

\usepackage[linkcolor=black, colorlinks=false, citecolor=black, filecolor=black]{hyperref}

\begin{document}

%
%

\twocolumn[
    \aistatstitle{Kernel Conditional Moment Constraints for Confounding Robust Inference}
    \aistatsauthor{ Kei Ishikawa \And Niao He }
    \aistatsaddress{
    ETH Z\"urich \\
    \href{mailto:kishikawa@student.ethz.ch}{kishikawa@student.ethz.ch}\\
\And
    ETH Z\"urich \\
    \href{mailto:niao.he@inf.ethz.ch}{niao.he@inf.ethz.ch}
} 
]

\begin{abstract}
We study policy evaluation of offline contextual bandits subject to unobserved confounders. 
Sensitivity analysis methods are commonly used to estimate the policy value under the worst-case confounding over a given uncertainty set. 
However, existing work often resorts to some coarse relaxation of the uncertainty set  for the sake of tractability, leading to overly conservative estimation of the policy value. 
In this paper, we propose a general estimator that provides a sharp lower bound of the policy value.
It can be shown that our estimator contains the recently proposed sharp estimator by \citet{dorn2022sharp} as a special case, and our method enables a novel extension of the classical marginal sensitivity model using f-divergence.
To construct our estimator, we leverage the kernel method to obtain a tractable approximation to the conditional moment constraints, which traditional non-sharp estimators failed to take into account. 
In the theoretical analysis, we provide a condition for the choice of the kernel which guarantees no specification error that biases the lower bound estimation.
Furthermore, we provide consistency guarantee of policy evaluation and extend the result to policy learning.
In the experiments with synthetic and real-world data, we demonstrate the effectiveness of the proposed method.
\end{abstract}


\section{INTRODUCTION}\label{chap:intro}

The offline contextual bandit is a simple but powerful model for decision-making with a wide range of applications such as data-driven personalized medical treatment, recommendations, and advertisements on online platforms.
In the evaluation of its policy value, the inverse probability weighting (IPW) method \citep{hirano2001estimation, hirano2003efficient} or its variant is commonly used.
This method relies on a so-called \emph{unconfoundedness} assumption, which essentially requires full observability of all  relevant variables so that there exist no unobserved variables that influence the selection of action and resulting reward \citep{rubin1974estimating}. 
However, in practice, such an assumption can easily be violated due to the existence of unobserved confounders that are not recorded in the logged data. 

A common way to address this problem is resorting to the worst-case lower bound of the policy value, namely, we minimize the policy value over a plausible uncertainty set that contains all the possible confounding situations. With such a lower bound, we can make an informed decision that is robust to confounding. 
The estimation and inference of such a lower bound are called sensitivity analysis and it has been extensively studied over the years  \citep{rosenbaum2002overt, tan2006distributional, rosenbaum2010design, liu2013introduction}.
Among a wide range of existing sensitivity models, a popular choice is the marginal sensitivity model by \citet{tan2006distributional} and its extensions.
Recently, \citet{zhao2019sensitivity} introduced an elegant algorithm for Tan's marginal sensitivity model using the linear fractional programming, which has revitalized the study of this model.
This approach was further extended to policy learning in \citet{kallus2018confounding, kallus2021minimax}.

However, these sensitivity analysis methods rely on algorithms using linear programming that finds an overly conservative lower bound of policy value. 
This is a fundamental problem, as these loose lower bound estimators are only guaranteed to be lower than or equal to the true lower bound of the uncertainty set, but they are not necessarily the consistent estimator of the true lower bound.
Even so, these algorithms have been widely adopted for their tractability.
To obtain a sharp lower bound, conditional moment constraints, which consist of infinite-dimensional linear constraints, must be leveraged. 
Recently, \citet{dorn2022sharp} analyzed these constraints and characterized a sharp lower bound of Tan's marginal sensitivity model using a conditional quantile function of the reward distribution. 
With this characterization, they proposed the first tractable algorithm to obtain the sharp lower bound that converges to the true lower bound of the policy value. 

In this paper, we address the same problem of sharp estimation from a new perspective.
Instead of using the conditional quantile function, we employ the kernel method \citep{scholkopf2002learning}, a rich and flexible modeling paradigm in machine learning.
We develop a tractable kernel approximation of the 
conditional moment constraints and propose an efficient algorithm to obtain the sharp lower bound.

\paragraph{Our contributions.} We summarize our contributions in several aspects below.

First, we extend the existing sensitivity analysis models by considering uncertainty sets characterized by more general convex constraints.
Our model includes the original sensitivity model by \citet{tan2006distributional} as the special case but it also includes a new f-sensitivity model 
that extends \citet{tan2006distributional}'s sensitivity model using f-divergence.

Second, we provide efficient algorithms based on the kernel method and low-rank approximation to obtain sharp estimators of the worst-case lower bound for the extended model.
Our new estimator is very general and it includes the previous sharp estimator by \citet{dorn2022sharp} as a special case.
Using the duality of the associated convex optimization problem, we further identify conditions for zero specification error guarantees and establish consistency guarantees of our estimator in policy evaluation.

Third, we show that our method can naturally be extended to policy learning, as it offers a very simple way to compute the policy gradient.
This is an advantage of our estimator compared to the previous sharp estimator \citep{dorn2022sharp}, which does not offer the possibility of policy learning.
We provide a consistency guarantee for policy learning with a sharp lower bound, which is similar to the guarantee for the non-sharp estimator by \citet{kallus2018confounding, kallus2021minimax}.

Last but not least, we demonstrate the effectiveness of imposing the kernel conditional moment constraints in several numerical experiments on both synthetic and real-world data.    We cover a wide range of problems in sensitivity analysis such as the generalized sensitivity models defined with f-divergence and policy learning, and our estimator consistently outperforms the conventional non-sharp estimators in these settings.

\paragraph{Related work.} Similar to our paper, \citet{kremer2022functional} used the kernel method for parameter estimation of models characterized by conditional moment restrictions.
They solved the dual of their original problem by using the dual representation of the $L_2$-norm of the conditional moment.
Though we solve a primal problem in this paper, we take great advantage of such a dual formulation in our theoretical analysis.
\citet{muandet2020kernel} considered hypothesis testing for conditional moment conditions.
They constructed their test statistic using a quadratic form of kernel matrix similar to the one we use.
The idea of using the kernel method to impose constraints has also been explored in other contexts such as fair regression \citep{perez2017fair}, distributionally-robust optimization \citep{staib2019distributionally}, worst-case risk quantification \citep{zhu2020worst}, and shape constraints to derivatives \citep{aubin2020hard}.
Recently, the kernel method has found various novel applications in causal inference, including instrumental variable regression \citep{singh2019kernel},
negative controls \citep{singh2020kernel, kallus2021causal, mastouri2021proximal},
and conditional mean squared error minimization for policy evaluation \citep{kallus2018balanced}.


\section{BACKGROUNDS AND PROBLEM SETTINGS}\label{chap:background}

\subsection{Confounded Offline Contextual Bandits}

Confounded offline contextual bandits are an extension of the standard offline contextual bandits that have an additional unobserved confounding variable.
We are interested in evaluating the value of policy $\pi$ from the offline data following base policy $\pi_\base$, which is generated according to the following model:
\begin{align}
\begin{array}{ll}
   X, U &\sim p(x, u), \\
   T|X, U &\sim \pi_\base(t|X, U),\\
   Y|T, X, U &\sim p(y|T, X, U),
\end{array}\label{eq:data_gen_base_confounded}
\end{align}
where only $Y$, $T$, and $X$ are observable and policy $\pi_\base(t|x, u)$ is unknown.
Action $t\in\cT$ is chosen by the (stochastic) base policy given context $X\in\cX$ and unobserved variable $U\in\cU$.
Reward $Y$ is randomly generated conditionally on the values of $T$, $X$, and $U$.

In the offline evaluation of policy $\pi$, we are interested in the expectation of $Y$ under the modified process of \eqref{eq:data_gen_base_confounded} where  $\pi_\base(t|x, u)$ is replaced by $\pi(t|x)$.
Thus, the desired policy value of $\pi$ can be written as 
\begin{align}
\begin{split}
V(\pi) &= \EE_{T\sim \pi(\cdot|X)}\left[ Y \right] \\
&= \EE_{T\sim \pi_\base(\cdot|X, U)}\left[ \left(\frac{\pi(T|X)}{\pi_\base(T|X, U)}\right) Y \right].
\end{split}
\label{eq:ipw_confounded}
\end{align}
Here, we only consider an observable policy $\pi(t|x)$, because it is trivially impossible to evaluate a policy that depends on unobserved variable $U$ only using the offline data.
For simplicity of notations, we denote $\EE_{T\sim \pi_\base(\cdot|X, U)}[f(Y, T, X, U)]$ as the expectations of $f(Y, T, X, U)$ under generative process \eqref{eq:data_gen_base_confounded} and $\EE_{T\sim \pi(\cdot|X)}[f(Y, T, X, U)]$  as its modification where $\pi_\base$ is replaced by $\pi$. 
Hereafter, we assume that $\pi_\base(T|X, U) > 0$ holds almost surely 
so that the inverse probability weights are always well-defined.

In unconfounded offline contextual bandits, we can use the inverse probability weighting (IPW) estimator 
\begin{equation}
    \hat V_\IPW(\pi) := \frac{1}{n}\sum_{i=1}^n \left( \frac{\pi(T_i|X_i)}{\hat \pi_\base(T_i|X_i)} \right) Y_i
    \label{eq:ipw}
\end{equation}
with estimated base policy $\hat \pi_\base$ to evaluate the policy $\pi$ consistently. 
However, when $\pi_\base$ depends on $U$, we can no longer construct such a consistent estimator, as the observable variables are only $Y$, $T$, and $X$, and any valid estimator must depend only on them.



To indicate a part of model \eqref{eq:data_gen_base_confounded} that can be approximated by the offline data, we use $p_\obs$ to indicate the observable distribution of \eqref{eq:data_gen_base_confounded} such that
\[
p_\obs(y, t, x) = \int p(y|t, x, u) \pi_\base(t|x, u) p(x, u)\rd u.
\]
Similarly, $p_\obs(t|x)$ and $p_\obs(x)$ denote the corresponding conditional and marginal distributions, and $\EE_\obs[f(Y, T, X)]$ represents the expectation of $f(Y, T, X)$ with respect to $p_\obs(y, t, x)$.
To represent the empirical average that approximates $\EE_\obs$, we use $\hat\EE_n$ so that $\hat\EE_n[f(Y, T, X)] := \frac{1}{n}\sum_{i=1}^n f(Y_i, T_i, X_i)$ for any $f(y, t, x)$.
Finally, we use abbreviation $\EE_\obs[f|t, x]$ to represent conditional expectation $\EE_\obs[f(Y, T, X)|T=t, X=x]$.
Hereafter, these observable distributions are assumed to be available for constructing estimators.

\subsection{Uncertainty Sets of Base Policies}\label{sec:policy_uncertainty_set}
A practical workaround to the above-mentioned issue is partial identification of policy value under some reasonable assumption about confounding.
More specifically, we first define some uncertainty set $\cE$ of $\pi_\base(t|x, u)$ in the form of constraint conditions.
Then we find the infimum  policy value $V_\textinf$ (or the supremum $V_\textsup$) within the uncertainty set as
\begin{equation}
V_\textinf(\pi) := \inf_{\pi_\base \in \cE} \EE_{T\sim\pi_\base(\cdot|X, U)}\left[\left(\frac{\pi(T|X)}{\pi_\base(T|X, U)} \right)Y\right].
\label{eq:v_inf_exact_policy}
\end{equation}
In the following, we list a few types of constraints used for the construction of the uncertainty sets.

\subsubsection*{Box Constraints}
The box constraints have been  widely adopted in the sensitivity analysis, and they can be written as 
\begin{equation}
    a_\pi(t, x) \leq \pi_\base(t|x, u) \leq b_\pi(t, x)
    \label{eq:box_policy_uncertainty_set}
\end{equation}
for some $a_\pi(t, x)$ and $b_\pi(t, x)$.
This assumption is used in the well-known marginal sensitivity model by \citet{tan2006distributional} as well as many of its extensions \citep{zhao2019sensitivity, kallus2018confounding, dorn2022sharp}.
\citet{tan2006distributional} considered a binary action space and assumed that the odds ratio of observational conditional probability $p_\obs(t|x)$ and the true base policy $\pi_\base(t|x, u)$ is not too far from $1$ so that
\begin{equation}
    \Gamma^{-1} \leq \frac{p_\obs(t|x) (1 - \pi_\base(t|x, u)}{(1 - p_\obs(t|x)) \pi_\base(t|x, u)} \leq \Gamma.
    \label{eq:tan_box_constraints}
\end{equation}
As $p_\obs(t|x)$ can be estimated from the observational data, we can enforce such constraints by choosing $a_\pi$ and $b_\pi$ in \eqref{eq:box_policy_uncertainty_set} as
$a_\pi(t, x) = 1/(1 + \Gamma (1/p_\obs(t|x) - 1))$
and $b_\pi(t, x) = 1/(1 + \Gamma^{-1} (1/p_\obs(t|x) - 1))$.

\subsubsection*{f-divergence Constraint}
The f-divergence is a measure of dissimilarity between two distributions.
For probability mass function (or density function) $p(t)$ and $q(t)$, the f-divergence between them is defined as 
$    D_f[p||q] := \sumint f \left( \frac{p(t)}{q(t)} \right)q(t)\rd t$
for some convex function $f:\RR\to\RR$ satisfying $f(1)=0$. \footnote{We use $\sumint$ because our method can handle both discrete and continuous treatment spaces in the same way.}
It is a rich class of divergence between probability distributions that includes many divergences such as the Kullback–Leibler (KL) divergence. 
Using the f-divergence, we introduce a new class of sensitivity assumption \footnote{\citet{jin2022sensitivity} proposed a similar uncertainty set based on f-divergence, but their model is different from ours. See more discussion in the supplementary material. }
\begin{equation}
    \EE_{X, U} \left[D_f[p_\obs(t|X) || \pi_\base(t|X, U)]\right] \leq \gamma,
    \label{eq:f_policy_uncertainty_set}
\end{equation}
where the expectation $\EE_{X, U}$ is taken with respect to $p(x, u)$ in \eqref{eq:data_gen_base_confounded} regardless of the policy.
By encoding the proximity of $\pi_\base(t|x, u)$ from $p_\obs(t|x)$ using the f-divergence instead of the box constraints, we can construct a flexible class of the uncertainty sets.
As we see later, this formulation is computationally convenient, 
as it can be expressed as simple expectation
\begin{align}
    &\EE_{X, U} \left[D_f[p_\obs(t|X) || \pi_\base(t|X, U)]\right]  \\
    & =  \EE_{T\sim\pi_\base(\cdot|X, U)}\left[ f\left(\frac{p_\obs(T|X)}{\pi_\base(T|X, U)}\right) \right].
    \label{eq:expected_f_divergence}
\end{align}
Compared to the marginal sensitivity model \citep{tan2006distributional} that imposes uniform bounds on the odds ratio $p_\obs(t|x)/\pi_\base(t|x, u)$ for any $x, u$, our f-sensitivity model upper bounds its average deviations from 1 (i.e. the unconfounded case).
Thus, when the odds ratio is locally very far from 1 around some x, u but is close to 1 elsewhere, the f-sensitivity model is a more reasonable choice than the conventional model.

\subsubsection*{Conditional f-constraint}
The above f-sensitivity model can be extended to an even more general case by letting convex function $f$ depend on $T$ and $X$.
Let $f:\RR\times\cT\times\cX\to\RR$ be a function satisfying that $f_{t, x}(\cdot):=f(\cdot, t, x)$ is convex and $f_{t, x}(1)=0$ for any fixed $t\in\cT$ and $x\in\cX$.
Then, we introduce a conditional f-constraint defined as
\begin{equation}
    \EE_{T\sim\pi_\base(\cdot|X, U)}\left[ f_{T, X}\left(\frac{p_\obs(T|X)}{\pi_\base(T|X, U)}\right) \right] \leq \gamma.
\label{eq:conditional_f_policy_uncertainty_set}
\end{equation}
Clearly, this uncertainty set generalizes f-divergence constraint \eqref{eq:f_policy_uncertainty_set}.
Moreover, this model contains the box constraint \eqref{eq:box_policy_uncertainty_set} as a special case.
By choosing 
\begin{equation}
    f_{t, x}(\tilde w) 
    = \begin{cases}
        0 &\text{ if } a_{\tilde w}(t, x) \leq \tilde w \leq b_{\tilde w}(t, x) \\
        \infty & \text{otherwise}
    \end{cases}
    \label{eq:f_box_constraints}
\end{equation}
for $a_{\tilde w}(t, x) = p_\obs(t|x) / b_\pi(t, x)$
and $b_{\tilde w}(t, x) = p_\obs(t|x) / a_\pi(t, x)$,
it becomes equivalent to box constraints \eqref{eq:f_box_constraints}.
To provide a systematic treatment of different types of uncertainty sets and unify the theoretical analysis, we will hereafter assume that the uncertainty sets of inverse probability weights always have some conditional f-constraint unless otherwise specified.

\subsection{Relaxed Uncertainty Sets of Inverse Probability Weights}\label{sec:weight_undertainty_set}

Let us introduce re-parametrization $w(y, t, x)= \EE_{T\sim\pi_\base(\cdot|X, U)}\left[\frac{1}{\pi_\base(T|X, U)}|y, t, x\right]$ to obtain tractable uncertainty sets.
The uncertainty set for $\pi_\base$ described earlier requires reparametrized weight $w(y, t, x)$ to satisfy the following two conditions:
\begin{equation}
    \EE_{T\sim\pi_\base(\cdot|X, U)}\left[ f_{T, X}\left(\frac{p_\obs(T|X)}{\pi_\base(T|X, U)}\right) \right] \leq \gamma
\end{equation}
and 
\begin{equation}
w(y,t,x) = \EE_{T\sim\pi_\base(\cdot|X, U)}\left[\frac{1}{\pi_\base(T|X, U)}|y, t, x\right]
\end{equation}
for some proper policy $\pi_\base$.
In general, both conditions are intractable.  Therefore, we will consider the relaxation of these conditions.

\subsubsection*{Relaxation of the Conditional f-constraint}
Let us first consider conditional f-constraint \eqref{eq:conditional_f_policy_uncertainty_set}.
With Jensen's inequality, we have 
\begin{align*}
    &\EE_{T\sim\pi_\base(\cdot|X, U)}\left[ f_{T, X}\left(\frac{p_\obs(T|X)}{\pi_\base(T|X, U)}\right) \right]\\
    &\geq \EE_{T\sim\pi_\base(\cdot|X, U)}\left[ f_{T, X}\left(
        \EE\left[\frac{p_\obs(T|X)}{\pi_\base(T|X, U)}|Y, T, X\right]
    \right) \right].
\end{align*}
Therefore, we can relax the condition \eqref{eq:conditional_f_policy_uncertainty_set} to
\begin{equation}
    \EE_\obs\left[ f_{T, X}\left( p_\obs(T|X)w(Y, T, X) \right) \right] \leq \gamma.
    \label{eq:relaxed_f_constraints}
\end{equation}


\subsubsection*{Relaxation of the Distributional Constraints}
Now we consider the relaxation of the second constraint, i.e.,  there exists proper underlying distribution $\pi_\base(t|x, u)$ that yields $w(y, t, x)$.
This constraint is usually relaxed to a one-dimensional linear constraint in previous work. Here, we present a tighter relaxation using infinite-dimensional linear constraints  called \emph{conditional moment constraints}. 
In the following, we present these two types of relaxation.

First, we describe the simple relaxation adopted in previous work \citep{zhao2019sensitivity, kallus2018confounding, kallus2021minimax}.
When action space $\cT$ is discrete and finite, the distributional constraint can be relaxed to
\begin{equation}
    \EE_\obs[\bbmone_{T=t} w(Y, T, X)] = 1 \text{ for any } t\in\cT
    \label{eq:zsb_constraint}
\end{equation}
and
\begin{equation}
  w(y, t, x) \geq 0 \text{ for any } y\in\cY, t\in\cT, \text{ and }x\in\cX
    \label{eq:nonnegativity_constraint}
\end{equation}
where $\bbmone$ denotes the indicator function for event $A$.
Following the naming convention in \citet{dorn2022sharp}, we will call constraint \eqref{eq:zsb_constraint} the ZSB constraint after the authors of \citet{zhao2019sensitivity}. 

Combining the above with the relaxation of the conditional f-constraint as in \eqref{eq:relaxed_f_constraints}, the following uncertainty set with the ZSB constraint can be defined:
\begin{equation}
    \cW_{f_{t, x}}^\ZSB := \left\{
    w \geq 0:
    \begin{array}{c}
        \EE_{X, U} \left[f_{T, X}\left(p_\obs(T|X)w\right)\right] \leq \gamma, \\
        \EE_\obs[\bbmone_{T=t}w] = 1 \text{ for any } t\in\cT
    \end{array}
    \right\}.
    \label{eq:zsb_uncertainty_set}
\end{equation}
For this uncertainty set, the associated lower bound is 
\begin{equation}
V_\textinf^\ZSB(\pi) := \inf_{w\in\cW_{f_{t, x}}^\ZSB}\EE_\obs\left[w(Y, T, X)\pi(T|X)Y \right].
\label{eq:v_inf_zsb}
\end{equation}

Now, we discuss the other relaxation based on conditional moment constraints.
It can be shown that it is possible to relax the distributional constraints to
\begin{equation}
    \EE_\obs[w(Y, T, X)|T=t, X=x] \cdot p_\obs(t|x) = 1
    \label{eq:conditional_moment_constraints}
\end{equation}
for any $t\in\cT$ and $x\in\cX$ plus a non-negativity constraint \eqref{eq:nonnegativity_constraint}.
The derivation of the above constraints is deferred to the supplementary material.
Note that, unlike the ZSB constraints, conditional moment constraints do not require that action space $\cT$ is discrete and finite.
We can again combine these conditional moment constraints (CMC) with relaxed conditional f-constraint \eqref{eq:relaxed_f_constraints} to obtain
\begin{equation}
    \cW_{f_{t, x}}^\CMC:= \left\{
    w \geq 0:
    \begin{array}{c}
        \EE_{X, U} \left[f_{T, X}\left(p_\obs(T|X)w\right)\right] \leq \gamma, \\
        \EE_\obs[w|t, x] \cdot p_\obs(t|x) = 1\\
        \text{ for any } t\in\cT \text{ and } x\in\cX
    \end{array}
    \right\}
    \label{eq:conditional_moment_uncertainty_set}
\end{equation}
and its corresponding lower bound
\begin{equation}
V_\textinf^\CMC(\pi):= \inf_{w\in\cW_{f_{t, x}}^\CMC}\EE_\obs\left[w(Y, T, X)\pi(T|X)Y \right]. \label{eq:v_inf_conditional_moment_constraints}\\
\end{equation}

Indeed, it can be shown that the conditional moment constraints are strictly sharper than the ZSB constraints as discussed in the supplementary material. 
Therefore, one can naturally obtain 
$\cW_{f_{t, x}} \subseteq \cW^\CMC_{f_{t, x}}\subseteq \cW^\ZSB_{f_{t, x}}$ and 
$V_\textinf \geq V_\textinf^\CMC \geq V_\textinf^\ZSB$.


\section{KERNEL CONDITIONAL MOMENT CONSTRAINTS}\label{chap:kernel_conditional_moment_constraints}

In this section, we introduce a empirical approximation of the conditional moment constraints using the kernel method \citep{scholkopf2002learning}.
The key idea is to approximate conditional moment $\EE_\obs[w(Y, T, X)p_\obs(T|X)|T=t, X=x]$ using the kernel ridge regression.
By constraining the estimated conditional moment to be close to $1$, we impose the conditional moment constraints to the empirical weight $\bsw=(w_1, \ldots, w_n)^T = \left(w(Y_1, T_1, X_1), \ldots, w(Y_n, T_n, X_n)\right)^T$.
In the following, we introduce three types of kernel conditional moment constraints (KCMC), namely, Gaussian process constraints, low-rank Gaussian process constraints, and low-rank hard constraints.
These kernel conditional moment constraints are all convex constraints, and they enable us to define a tractable uncertainty set and an associated estimator of lower bound as
\begin{equation}
\hat\cW^\KCMC_{f_{t, x}}
= \left\{w\geq 0: 
\begin{array}{c}
    \hat\EE_n \left[f_{T, X}\left( p_\obs(T|X)w\right)\right] \leq \gamma, \\
    w(y, t, x)\text{ satisfies the KCMC}
\end{array}
\right\}
    \label{eq:kernel_empirical_uncertainty_set}
\end{equation}
and
\begin{equation}
    \hat V_\text{inf}^\KCMC := \min_{\bsw\in\hat\cW^\KCMC_{f_{t, x}}}\hat\EE_n[w(Y, T, X)\pi(T|X)Y].
    \label{eq:v_inf_kernel_empirical}
\end{equation}

\subsection{Gaussian Process Constraints}
In this subsection, we derive our first kernel conditional moment constraints, which we call the Gaussian process constraints.
We begin by formally formulating the idea of using the kernel ridge regression for the conditional moment constraints and then motivate its interpretation as a Gaussian process to obtain reasonable kernel conditional moment constraints.

\subsubsection*{Estimation of Conditional Expectation by Kernel Ridge Regression}
Let us introduce kernel $k:(\cT\times\cX) \times (\cT\times\cX)\to \RR$ with associated reproducing kernel Hilbert space (RKHS) $\cH$ of functions $h:\cT\times\cX\to\RR$, inner product $\langle\cdot, \cdot\rangle_\cH$, and norm $\|\cdot\|_\cH$.
Let us further introduce re-parametrization 
$$e(y, t, x) := p_\obs(t|x)w(y, t, x) - 1,$$
so that conditional moment constraints \eqref{eq:conditional_moment_constraints} can be written as 
$$\EE_\obs[e(Y, T, X)|T=t, X=x] = 0$$
for any $t\in\cT$ and $x\in\cX$.

Then, using the kernel ridge regression, one can estimate conditional expectation $g(t, x):=\EE_\obs[e(Y, T, X)|T=t, X=x]$ as 
\begin{equation*}
     \hat g = \arg\min_{g\in\cH} \hat\EE_n |g(T, X) - e(Y, T, X)|^2 + \sigma^2\|g\|_\cH^2.
\end{equation*}
for some $\sigma^2>0$.
The above problem yields an analytical solution, and we can get
\[
  \hat \bsg = K(K+\sigma^2 I_n)^{-1}\bse
\]
for $\hat \bsg := \left(\hat g(T_1, X_1), \ldots, \hat g(T_n, X_n)\right)^T$ and $\bse := \left(e(Y_1, T_1, X_1), \ldots, e(Y_n, T_n, X_n)\right)^T$.
Here, $K$ denotes the kernel matrix such that $K_{i,j} = k((T_i, X_i), (T_j, X_j))$ and $I_n$ is the identity matrix of order $n$.
To impose the conditional moment constraints, we can consider the constraint
$\hat \bsg \approx \bszero$.
Note that we cannot impose the exact equality, i.e., 
$K(K+\sigma^2 I_n)^{-1}\bse = \bszero$,
as this leads to only solution $\bsw = \bsone / \bsp_{\obs, T|X}$,
\footnote{Here, $\bsp_{\obs, T|X}:=\left(p_\obs(T_1|X_1), \ldots, p_\obs(T_n|X_n)\right)^T$ and $\bsone:=(1, \ldots, 1)\in\RR^n$. The division is taken element-wise.}
whose resulting policy value estimator is exactly the confounded IPW estimator.
Therefore, we need to find a reasonable way to impose this close-to-zero constraint.

\subsubsection*{Construction of Uncertainty Set With a Credible Set of Gaussian Process}
Now, by interpreting the kernel ridge regression as the Gaussian processes regression \citep{rasmussen2003gaussian}, we make an intuitive association of the close-to-zero constraint with the credible set in Bayesian statistics.
Consider the following Gaussian process regression model \citep{rasmussen2003gaussian}:
\begin{equation}
\begin{array}{l}
    g \sim \mathcal{GP}(0, k(\cdot, \cdot)), \;
    \bse = \bsg + \varepsilon, \;
    \varepsilon \sim\mathcal{N}(\bszero, \sigma^2I_n),
\end{array}
\label{eq:gaussian_process_regression_model}
\end{equation}
where $\bsg := \left(g(T_1, X_1), \ldots, g(T_n, X_n)\right)$.
$\mathcal{GP}$ and $\mathcal{N}$ indicate the Gaussian process and the multivariate normal distribution with the specified mean and covariance parameters. 
Under this model, the posterior of $\bsg$ given $\bse$ is a multivariate normal distribution with mean and variance $\mu_{\bsg|\bse} = K(K + \sigma^2 I_n)^{-1} \bse = \hat \bsg$ and $\Sigma_{\bsg|\bse} = K - K(K + \sigma^2 \mathrm{I})^{-1}K$. 
For the posterior of $\bsg$, we now define a $(1 - \alpha)$ credible set.
As the posterior of $\bsg$ given $\bse$ is a multivariate normal distribution, we can take highest posterior density set
\begin{equation}
C_{\bsg|\bse} = \{\bsg : (\bsg - \mu_{\bsg|\bse})^T \Sigma_{\bsg|\bse}^{-1} (\bsg - \mu_{\bsg|\bse}) \leq \chi^2_{n}(1 - \alpha)\}.
\label{eq:gp_credible_set}
\end{equation}
as the credible set.
Here, $\chi^2_n(1-\alpha)$ denotes the $(1-\alpha)$-percentile of $\chi^2$ distribution with degrees of freedom $n$.
Now, as we want estimated conditional moment function $g(t, x)$ to be close to zero, we can require that $\bsg=\bszero$ to be included in the $(1-\alpha)$ credible set and obtain the Gaussian process kernel conditional moment constraints as follows:
\begin{align}
\begin{split}
    \KCMC_\text{GP}
    &\stackrel{\text{def}}{\Leftrightarrow} \bszero \in  C_{\bsg|\bse} \\
    &\Leftrightarrow \bse^T M_\mathrm{GP} \bse \leq \chi^2_{n}(1 - \alpha) \\
\end{split}
\label{eq:kernel_conditional_moment_constraints_gaussian_process}
\end{align}
where $M_\text{GP}:=(K+\sigma^2I_n)^{-1}K(K - K(K+\sigma^2I_n)^{-1}K)^{-1}K(K+\sigma^2I_n)^{-1}$.
As $\bse=\bsp_{\obs, T|X} \odot \bsw - \bsone$,\footnote{The element-wise product operator is denoted by $\odot$.}
this is a quadratic constraint for $\bsw$, which makes it possible to compute the associated lower bound by solving convex programming.

\subsection{Low-rank Gaussian Process Constraints} 

A practical downside of the above formulation is the linear growth of the number of constraints to sample size $n$.
Additionally, the calculation of the matrix inverse takes $\cO(n^3)$, which can be impractically slow for a large sample size.
To mitigate these issues, we propose the use of low-rank approximation to the kernel matrix.

Let's consider spectral decomposition of the kernel matrix, $K=V\Lambda V^T$ with orthonormal matrix $V\in\RR^{n\times n}$ and diagonal matrix $\Lambda = \mathrm{diag}(\lambda_1, \ldots, \lambda_n)\in \RR^{n\times n}$ satisfying $\lambda_1 \geq \lambda_2 \geq \ldots \geq \lambda_n \geq 0$.
By truncating the spectrum after the first $D$ dimensions, we approximate the kernel matrix as $K\approx \tilde K := \tilde V \tilde \Lambda {\tilde V}^T$, where $\tilde \Lambda = \mathrm{diag}(\lambda_1, \ldots, \lambda_D)\in \RR^{D\times D}$ and $\tilde V \in \RR^{n\times D}$ is the first $D$ columns of matrix $V$.
Then, we can substitute $\tilde K$ in place of $K$ to obtain the approximate posterior of $\bsg$, which is a  multivariate normal distribution with mean and variance $\tilde\mu_{\bsg|\bse} = \tilde K( \tilde K + \sigma^2 I_n)^{-1} \bse$ and $\tilde\Sigma_{\bsg|\bse} = \tilde K - \tilde K(\tilde K + \sigma^2 \mathrm{I})^{-1}\tilde K$.
Thus, we can analogously define the credible set and the kernel conditional moment constraints as
\begin{equation}
\tilde C_{\bsg|\bse} = \{\bsg : (\bsg - \tilde\mu_{\bsg|\bse})^T \tilde\Sigma_{\bsg|\bse}^{-1} (\bsg - \tilde\mu_{\bsg|\bse}) \leq \chi^2_{D}(1 - \alpha)\}.
\label{eq:low_rank_gq_credible_set}
\end{equation}
and
\begin{align}
\begin{split}
    \KCMC_\text{low-rank GP}
    &\stackrel{\text{def}}{\Leftrightarrow} \bszero \in \tilde C_{\bsg|\bse} \\
    &\Leftrightarrow \bse^T M_\text{low-rank GP} \bse \leq \chi^2_{D}(1 - \alpha) \\
\end{split}
\label{eq:kernel_conditional_moment_constraints_low_rank_gaussian_process}
\end{align}
where $M_\text{low-rank GP}:=(\tilde K+\sigma^2I_n)^{-1}\tilde K(\tilde K - \tilde K(\tilde K+\sigma^2I_n)^{-1}\tilde K)^{-1}\tilde K(\tilde K+\sigma^2I_n)^{-1}$.

One big difference of the credible set of original Gaussian process \eqref{eq:gp_credible_set} and low-rank Gaussian process \eqref{eq:low_rank_gq_credible_set} is the degree of freedom of the $\chi^2$ distribution.
We will discuss the reason for it in the supplementary material, but this is essentially due to the fact that the distribution of the $\bsg$ is restricted to the $D$-dimensional subspace spanned by columns of $\tilde V$.

\subsection{Low-rank Hard Constraints}

Indeed, when we use the low-rank approximation, it is possible to impose the hard constraint to the solution of the low-rank kernel ridge regression as $\hat \bsg = \tilde K (\tilde K + \sigma^2 I_n) \bse = 0$ without suffering from the issue of feasibility set reducing to singleton $\{\bsone/\bsp_{\obs, T|X}\}$. 

Using the interpretation of the spectral decomposition as the kernel principal component analysis (PCA) \citep{scholkopf1997kernel}, the above hard constraints can be reformulated as the following empirical orthogonality condition 
\[
    \hat \EE_n[e(Y, T, X)\bsphi^\KPCA(T, X)] = \bszero
\]
of error $e(Y, T, X)$ and kernel principal components 
$\bsphi^\KPCA(t, x):= \left( \phi_1^\KPCA, \ldots, \phi_D^\KPCA \right)^T$
obtained of the kernel PCA applied to $\{T_i, X_i\}_{i=1}^n$.
This condition is obviously a relaxation of the conditional moment constraints \eqref{eq:conditional_moment_constraints}, as they imply orthogonality condition $\EE_\obs[e(Y, T, X) \psi(T, X)]=0$ for any $\psi(t, x)$. 
By slightly generalizing the above condition, we introduce the following low-rank hard constraints.
\begin{equation}
    \KCMC_\text{low-rank orth}
    \stackrel{\text{def}}{\Leftrightarrow} \hat \EE_n[e(Y, T, X)\bspsi(T, X)] = \bszero,
    \label{eq:kernel_conditional_moment_constraints_low_rank_orthogonality}
\end{equation}
for orthogonal function class $\bspsi(t, x):=\left(\psi_1, \ldots, \psi_D\right)^T$.

In our theoretical analysis, constraints \eqref{eq:kernel_conditional_moment_constraints_low_rank_orthogonality} provide the most suitable estimator for our studies.
This is because choosing $\{\phi_d\}_{d=1}^D$ independently from the samples enables us to decouple the discussion on the goodness of the constraints and the goodness of the policy value estimator.

In practice, these low-rank estimators of the lower bound have a trade-off between sharpness and credibility. Low-rank Gaussian process constraints \eqref{eq:kernel_conditional_moment_constraints_low_rank_gaussian_process}
only impose soft quadratic constraints and can sometimes produce a too pessimistic lower bound.
On the other hand, low-rank hard constraints \eqref{eq:kernel_conditional_moment_constraints_low_rank_orthogonality} 
produce a tighter lower bound but require careful selection of the number of constraints, as excessively strong constraints would lead to a too optimistic estimate. 

\section{THEORETICAL ANALYSIS}\label{chap:theoretical_analysis}

In this section, we study the property of the kernel conditional moment constraints for confounding robust inference.
For the convenience of the analysis, we only consider low-rank  orthogonality condition \eqref{eq:kernel_conditional_moment_constraints_low_rank_orthogonality} as the kernel conditional moment constraints in this section.
Thus, we will focus on the property of the following population lower bound 
\begin{equation}
    V_\textinf^\KCMC(\pi) = \inf_{w\in\cW^\KCMC_{f_{t, x}}}\EE_\obs[w(Y, T, X)\pi(T|X)Y]
    \label{eq:v_inf_kernel_conditional_moment_constraints}
\end{equation}
for
\begin{equation}
    \cW^\KCMC_{f_{t, x}} = \left\{
    w\geq 0:
    \begin{array}{c}
         \EE_\obs[f_{T, X}(wp_\obs(T|X))] \leq \gamma, \\ 
         \EE_\obs[\left(wp_\obs(T|X) - 1\right)\bspsi] = \bszero  \\
    \end{array}
    \right\}
\end{equation}
and its empirical version 
\begin{equation}
    \hat V_\textinf^\KCMC(\pi) = \inf_{\bsw\in\hat\cW^\KCMC_{f_{t, x}}}\hat\EE_n[w(Y, T, X)\pi(T|X)Y]
    \label{eq:empirical_v_inf_kernel_conditional_moment_constraints}
\end{equation}
for
\begin{equation}
    \hat\cW^\KCMC_{f_{t, x}} = \left\{
    w \geq 0:
    \begin{array}{c}
         \hat\EE_n[f_{T, X}(wp_\obs(T|X))] \leq \gamma, \\
         \hat\EE_n[\left(wp_\obs(T|X) - 1\right)\bspsi] = \bszero  \\
    \end{array}
    \right\}.
\end{equation}
In the theoretical analysis, we characterize the properties of our estimator with the dual solution of our original problem.
We first analyze specification error of the kernel conditional moment constraints $\left|V^\CMC_\textinf - V^\KCMC_\textinf\right|$ and provide a condition on orthogonal function class $\{\psi_d(t, x)\}_{d=1}^D$ under which the specification error becomes zero.
Additionally, we study empirical estimator $\hat V_\textinf^\KCMC$ and prove consistency guarantees for policy evaluation and learning.

Due to the space limitation, we will defer the proofs and the precise assumptions for the following statements to the supplementary material.

\subsection{Specification Error}

First, we present the condition under which the specification error of estimator $V_\textinf^\KCMC(\pi)$ becomes zero for policy $\pi$ such that $\left| V_\textinf^\KCMC(\pi) - V_\textinf^\CMC(\pi) \right| = 0$.

\begin{theorem}[No specification error]\label{th:no_specification_error}
Let $\eta^*_\CMC:\cT\times\cX\to\RR$  be the solution to the dual problem of \eqref{eq:v_inf_conditional_moment_constraints}. 
Then $V_\textinf^\KCMC(\pi)$ has zero specification error if
\begin{equation}
    \eta_\CMC^* \in \mathrm{span}\left(\{\psi_1, \ldots, \psi_D\}\right).
    \label{eq:eta_cmc_in_kernel_subspace}
\end{equation}
\end{theorem}

Using this lemma, it is possible to prove that the previous sharp estimator for box constraints by \citet{dorn2022sharp} is a special case of estimator that uses our kernel conditional moment constraints.
They identified the analytical form for optimal orthogonal function class $\{\psi_1\} = \{\eta^*_\CMC\}$ to derive a one-dimensional linear constraint to impose the conditional moment constraints, which are originally infinite-dimensional.
More details on their estimator are discussed in the supplementary material.

\subsection{Consistency of Policy Evaluation and Learning}

Now, we study empirical estimator $\hat V_\textinf^\KCMC$ and provide consistency guarantees for policy evaluation and learning in the case of finite-dimensional concave policy class.
We prove both consistency results by a reduction of our problem to the M-estimation \citep{van2000empirical} using the dual formulation.

Let us define $L_{\theta, \pi}:\cY\times\cT\times\cX\to\RR$ as the loss function of the dual objective of \eqref{eq:v_inf_kernel_conditional_moment_constraints} 
so that the dual problem becomes
$\max_{\theta\in\Theta}\EE[-L_{\theta, \pi}(Y, T, X)]$
for dual parameter $\theta\in\Theta$.
Additionally, let us introduce the following assumption:
\begin{assumption}[Regularity of loss function]\label{assum:regular_loss}
Loss function $\ell_\theta:\cY\times\cT\times\cX\to\RR$ satisfies
1) $\theta\mapsto\ell_{\theta}$ is continuous,
2) $\EE\left|\ell_{\theta}\right|<\infty$ for any $\theta\in\Theta$,
3) $\theta_0\in\arg\min_{\theta\in\Theta}\EE[\ell_{\theta}]$ is unique,
and 4) $\EE[G_\varepsilon]<\infty$ for $G_\varepsilon:=\sup_{\theta\in\Theta:\ \|\theta - \theta_0\| \leq \varepsilon} \left| \ell_{\theta} \right|$ for some $\varepsilon > 0$.
\end{assumption}
With this assumption, we can immediately show the consistency guarantee for policy evaluation:
\begin{theorem}[Consistency of policy evaluation]\label{th:policy_evaluation_consistency}
For fixed policy $\pi$, if $\ell_\theta:=L_{\theta, \pi}$ satisfies Assumption \ref{assum:regular_loss}, then we have $\hat V_\textinf^\KCMC(\pi)  \pto V_\textinf^\KCMC(\pi)$.
\end{theorem}

The above theorem can be extended to policy learning,
by considering joint parameter space $\Theta':=\Theta\times\mathcal{B}$ of $\theta':=(\theta, \beta)$ for policy class $\{\pi_\beta(t|x):\ \beta\in\mathcal{B}\}$ and
introducing joint loss function $L'_{\theta'} := L_{\theta, \pi_\beta}$ so that 
$\max_{\beta\in\mathcal{B}}V_\textinf^\KCMC(\pi_\beta)
= \max_{\theta'\in\Theta'}\EE[-L'_{\theta'}(Y, T, X)]$.
\begin{theorem}[Consistency of concave policy learning]\label{th:policy_learning_consistency}
Assume the policy class is concave so that $\beta\mapsto\pi_\beta(t|x)y$ is concave.
Define $\beta_0 \in \arg\max_{\beta\in\mathcal{B}} V_\textinf^\KCMC(\pi_\beta)$ 
and its estimator $\hat\beta \in \arg\max_{\beta\in\mathcal{B}} \hat V_\textinf^\KCMC(\pi_\beta)$.
If $L'_{\theta'}$ satisfies Assumption \ref{assum:regular_loss}, 
then, we have
$\hat V_\textinf^\KCMC(\pi_{\hat\beta}) \pto V_\textinf^\KCMC(\pi_{\beta_0})$.
\end{theorem}

An example of concave policy is mixed policy $\pi_\beta(t|x):=\sum_k\beta_k\pi_k(t|x)$ for $\sum_k\beta_k=1$, $\beta_k\geq0$. Indeed, policy learning with such a concave policy class is concave; therefore, the globally optimal policy can be found by convex optimization algorithms.


\section{NUMERICAL EXPERIMENTS}\label{chap:experiments}

In this section, we present numerical examples to compare our estimators with the existing estimators.
In addition to the standard policy evaluation, we also consider policy learning and the f-sensitivity model.

\begin{figure*}[!tb]
    \centering
    \begin{subfigure}[b]{0.32\textwidth}
        \centering
        \includegraphics[width=\linewidth,trim={5 0 0 11mm},clip]{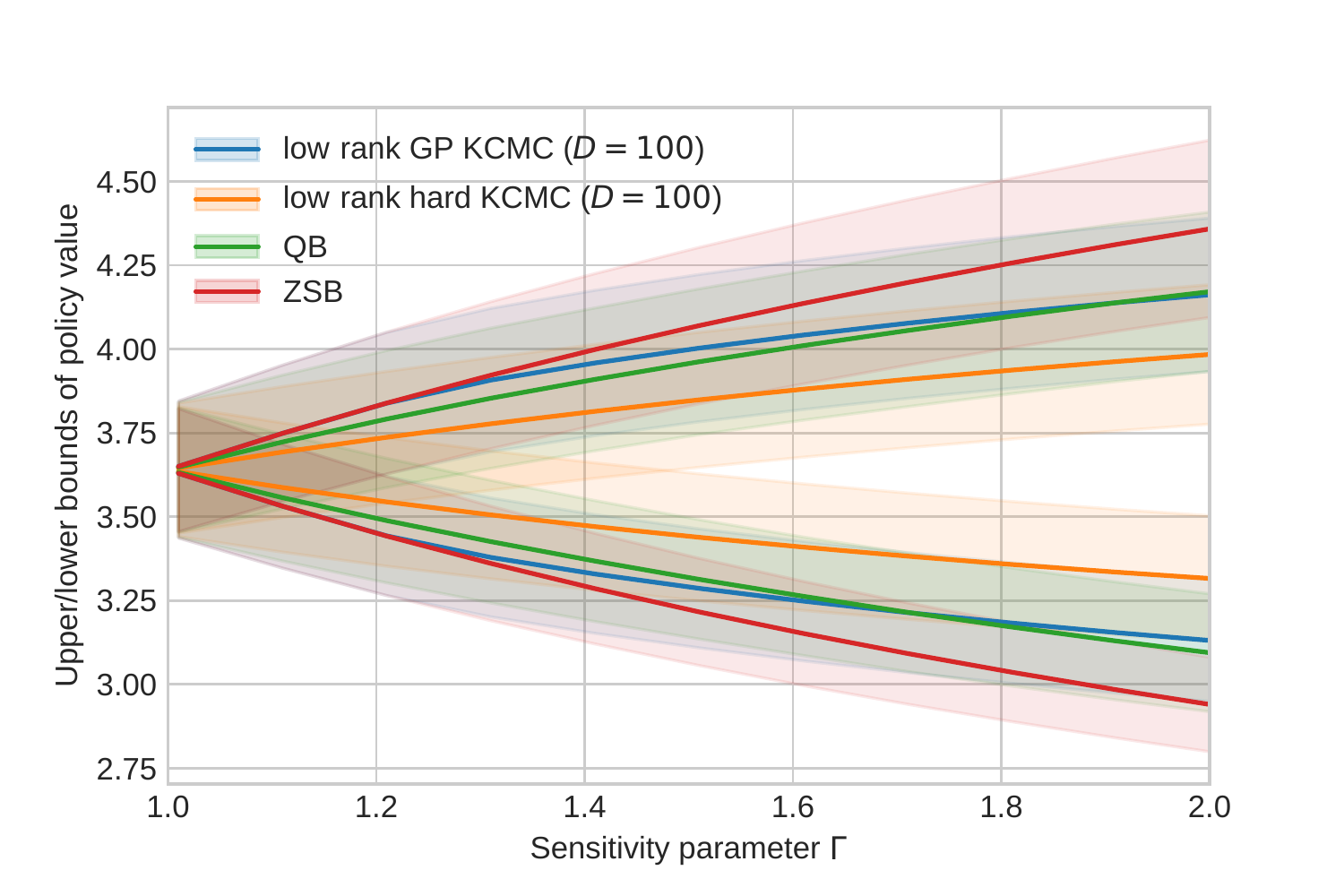}
        \caption{}
        \label{fig:policy_evaluation_synthetic_binary_changing_lambda}
    \end{subfigure}
    \begin{subfigure}[b]{0.32\textwidth}
        \centering
        \includegraphics[width=\linewidth, trim={5 0 0 11mm},clip]{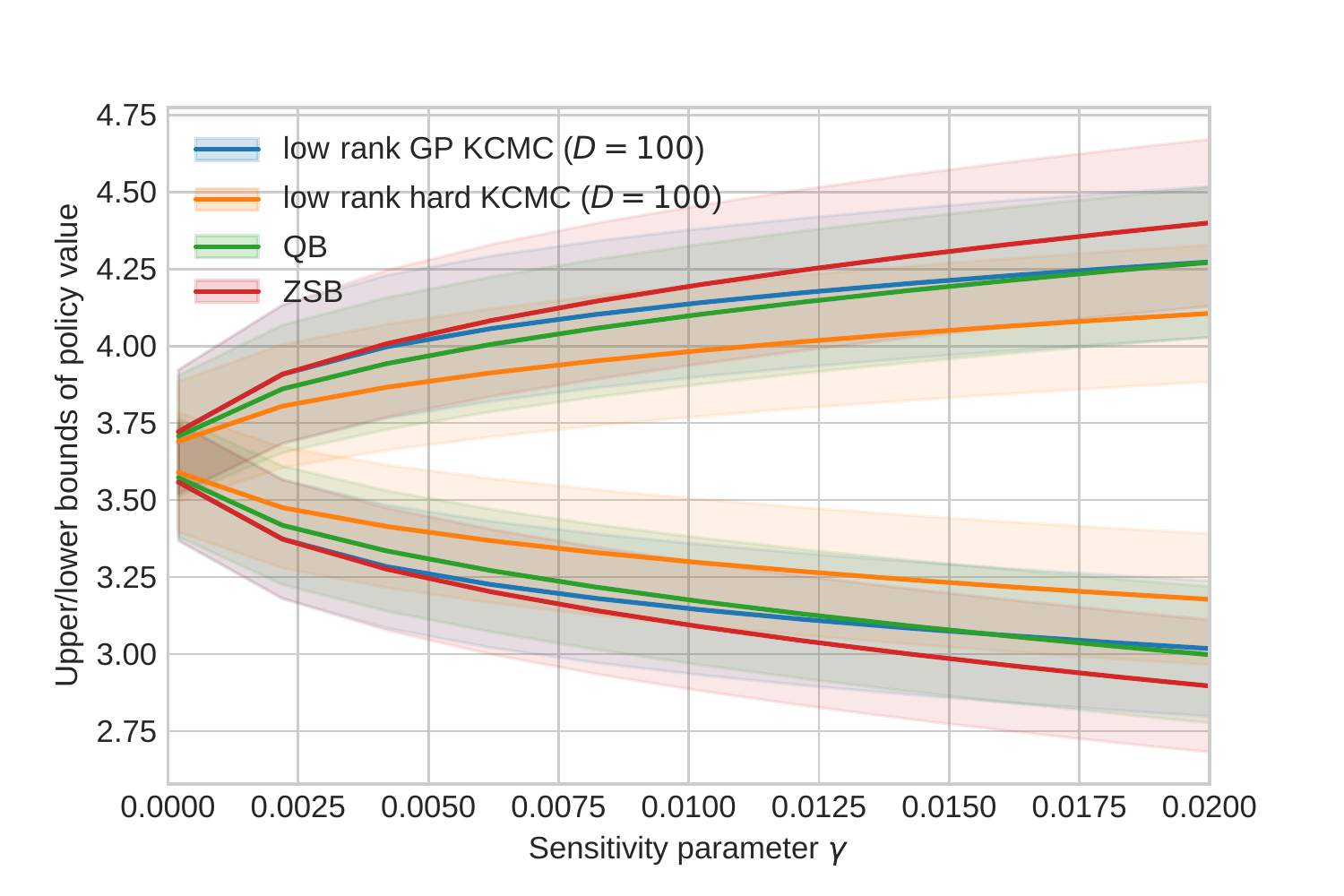}
        \caption{}
        \label{fig:policy_evaluation_synthetic_binary_changing_gamma_KL}
    \end{subfigure}
    \begin{subfigure}[b]{0.32\textwidth}
        \centering
        \includegraphics[width=\linewidth, trim={5 0 0 11mm},clip]{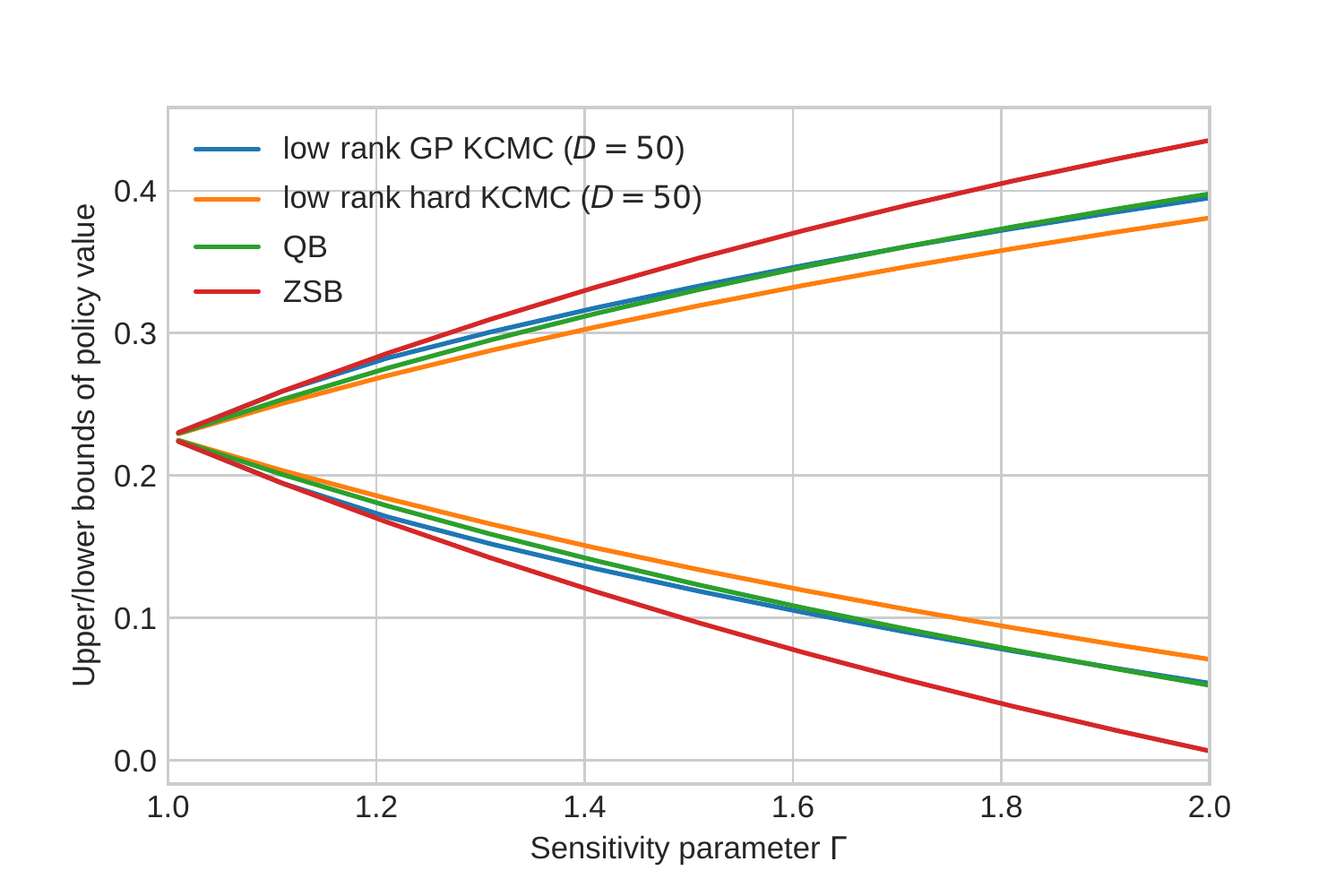}
        \caption{}
        \label{fig:policy_evaluation_real_binary_changing_lambda}
    \end{subfigure}
    \caption{Estimated upper and lower bounds using different types of estimators for sensitivity analysis: (a) Tan's marginal sensitivity model for policy value on synthetic data; (b) The KL sensitivity model for policy value on synthetic data; (c) Tan's marginal sensitivity model ($\Gamma=1.5$) for average treatment effect on NLS data. }
\end{figure*}

\subsection{Experimental Settings}

In the first three experiments, we use the synthetic data adapted from \citet{kallus2018confounding, kallus2021minimax}.
We repeat the experiment 10 times with different random seeds and report the mean and one standard deviation range by a line and a band around it.
The last experiment uses subsamples of data from the 1966-1981 National Longitudinal Survey (NLS) of Older and Young Men, which was also used in \citet{dorn2022sharp}.

In the experiments, four types of estimators are compared.
As the baseline, we consider the conventionally used ZSB estimator which solves the empirical version of \eqref{eq:v_inf_zsb}.
To this baseline, we compare the proposed estimators based on two types of kernel conditional moment constraints (KCMC), which are
\eqref{eq:kernel_conditional_moment_constraints_low_rank_gaussian_process} 
and  \eqref{eq:kernel_conditional_moment_constraints_low_rank_orthogonality}.
We call them the low-rank GP KCMC and the low-rank hard KCMC, respectively.
For low-rank hard KCMC, the orthogonal function class was chosen by the kernel PCA.
Lastly, as a reference, we include the quantile balancing (QB) estimator by \citet{dorn2021doubly}, which is a special case of low-rank hard KCMC.
More details on the experimental settings and additional results can be found in the supplementary material.\footnote{The code can be found at \url{https://github.com/kstoneriv3/confounding-robust-inference}.}

\vspace{-2mm}
\subsection{Policy Evaluation}
\vspace{-2mm}
Figure \ref{fig:policy_evaluation_synthetic_binary_changing_lambda} compares the tightness of the bounds obtained by different estimators in policy evaluation.
Clearly, the sharper estimators (KCMC and QB) are producing tighter bounds than the ZSB estimator.
Here, we can see that the low-rank hard KCMC's bounds are much tighter than those of the other sharp estimators.
This exemplifies the aforementioned trade-off between the sharpness and the credibility of bounds obtained by the soft and hard KCMC.

\vspace{-2mm}
\subsection{Extension to f-divergence Sensitivity Model}
\vspace{-2mm}
Next, to illustrate application of the f-sensitivity models \eqref{eq:f_policy_uncertainty_set},
we present an example of the KL-sensitivity model in Figure \ref{fig:policy_evaluation_synthetic_binary_changing_gamma_KL}.
We can see that the KL-sensitivity model can provide continuous control of the level of confounding by the sensitivity parameter, similarly to Tan's marginal sensitivity model.

\vspace{-2mm}
\subsection{Extension to Policy Learning}
\vspace{-2mm}
Figure \ref{fig:policy_learning_synthetic_binary} shows the learning curves during the max-min policy optimization with Tan's marginal sensitivity model.
Though order $\hat V^\KCMC_\textinf \geq \hat V^\ZSB_\textinf$ is still maintained, the KCMC lower bounds estimated from training data are actually higher than the ground truth.
This can be interpreted as an overfitting phenomenon in the joint maximization of learning and the inner dual problem.
Since $V_\textinf^\CMC$ estimated by test data is still lower than the ground truth,  improvement of policy learning may be possible with more careful control of test errors such as cross validation.

\begin{figure}[!htb]
    \centering
    \includegraphics[width=\linewidth, trim={5 0 0 11mm},clip]{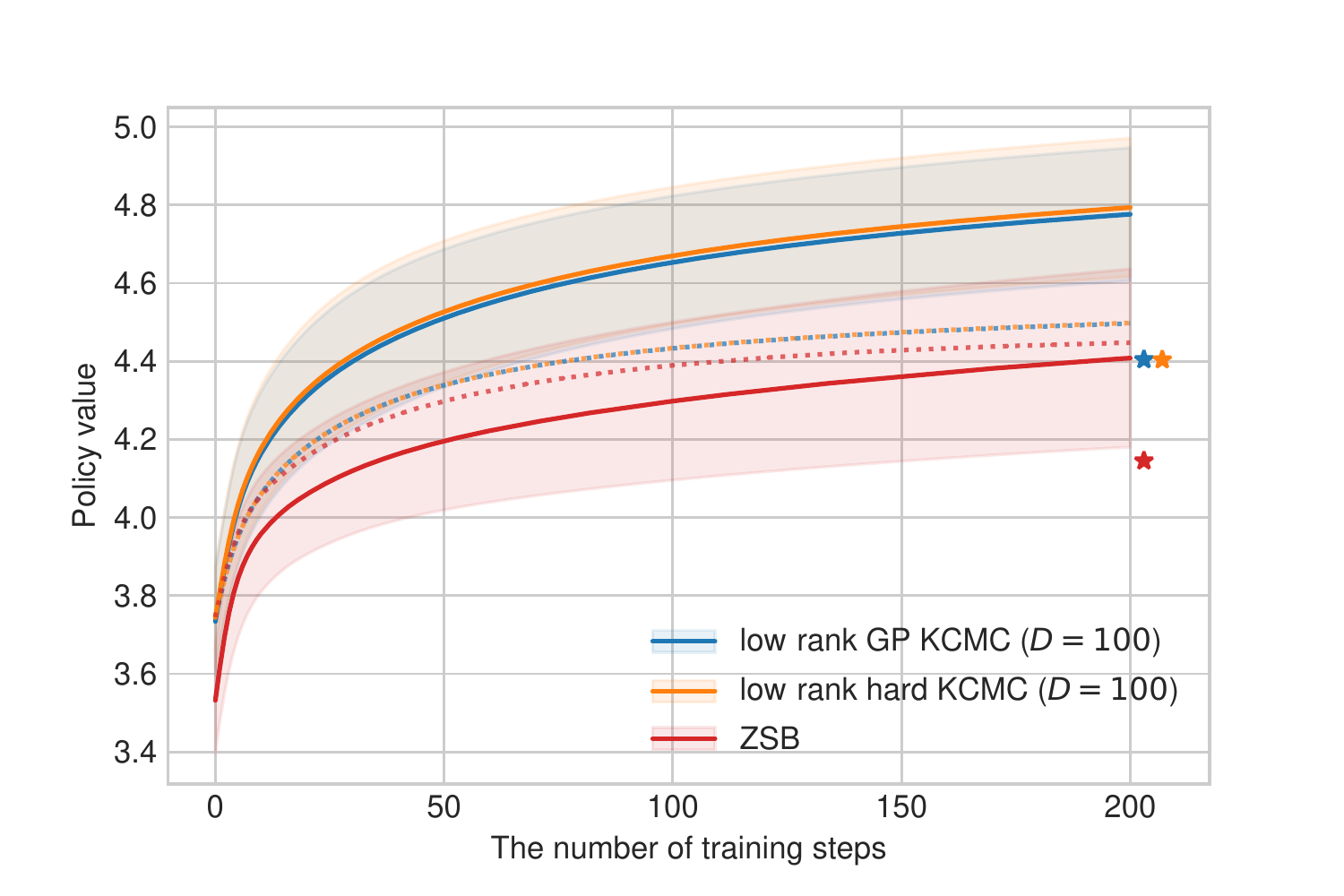}
    \caption[Caption for LOF]{The value of $\hat V_\textinf$ on training data during the policy learning. The dotted lines represent ground truth policy values $V$. The star symbols ($\star$) at the end of the learning curves indicate the average lower bound $V_\textinf^\CMC$ of the learned policy. 
    \protect\footnotemark
    }
    \label{fig:policy_learning_synthetic_binary}
\end{figure}

\footnotetext{
    Ground truth $V$ was estimated by unconfounded Monte Carlo simulation of the true data-generating process. 
    Lower bound $V_\textinf^\CMC$ was estimated by the low-rank hard KCMC estimators on test data.
}

\vspace{-2mm}
\subsection{Treatment Effect Estimation on NLS Data}
\vspace{-2mm}
Lastly, in Figure \ref{fig:policy_evaluation_real_binary_changing_lambda}, 
we show the upper and lower bounds of the average treatment effect ($\EE[Y|T=1] - \EE[Y|T=0]$) of union membership ($T$) on log wages ($Y$) estimated from the NLS dataset.
Similarly to the previous examples, the low-rank GP KCMC and quantile balancing estimates are very close while the ZSB and the low-rank hard KCMC produce looser and (possibly overly) tighter bounds.


\section{CONCLUSION}\label{chap:conclusion}

In this paper, we proposed kernel approximation of the conditional moment constraints to achieve sharp and general sensitivity analysis.
We theoretically studied the property of the kernel conditional moment constraints and established consistency guarantees for policy evaluation and learning.
We also confirmed the effectiveness of our approach empirically, with numerical examples covering various types of problems in the sensitivity analysis.


\subsubsection*{Acknowledgements}
K.I. was supported by the Heiwa Nakajima Foundation. N.H. was supported by ETH Research Grant and Swiss National Science Foundation. 
The authors acknowledge Takafumi Kanamori (Tokyo Institute of Technology) for pointing out a critical technical error in an early version of the draft. 

\subsubsection*{References}
\bibliography{references}

\appendix
\onecolumn

\section{Derivation of ZSB Constraints and Conditional Moment Constraints}
Here, we discuss the two ways to relax the distributional constraint of the original uncertainty set for base policy $\pi_\base$.
The condition that $\pi_\base(t|x, u)$ is a proper distribution is equivalent to 
\begin{equation}
   \sumint_\cT \pi_\base(t'|x, u) \rd t' = 1 \text{ and } \pi_\base(t|x, u) \geq 0
   \text{ for any }
   t\in\cT, x\in\cX, \text{ and }u\in\cU.
   \label{eq:proper_base_policy}
\end{equation}
These constraints have traditionally been relaxed to the ZSB constraints, which are used in the well-known H\'ajek estimator \citep{zhao2019sensitivity, kallus2018confounding, kallus2021minimax}.
However, there is also a tighter relaxation called conditional moment constraints, which we employ in our work.
In the following, we present these two types of relaxation.

\subsection{ZSB Constraint}
When action space $\cT$ is discrete and finite, a well-known relaxation of \eqref{eq:proper_base_policy} is
\begin{equation}
    \EE_\obs[\bbmone_{T=t}w(Y, T, X)] = 1 \text{ for any } t\in\cT
    \label{app-eq:zsb_constraint}
\end{equation}
and
\begin{equation}
    w(y, t, x) \geq 0 \text{ for any } y\in\cY, t\in\cT, \text{ and }x\in\cX,
    \label{app-eq:nonnegativity_constraint}
\end{equation}
where $\bbmone$ denotes the indicator function for event $A$.
Following the naming convention in \citet{dorn2022sharp}, we will call this constraint the ZSB constraint after the authors of \citet{zhao2019sensitivity}.
Condition \eqref{app-eq:zsb_constraint} can be obtained as
\begin{align*}
    \EE_\obs[w(Y, T, X)\bbmone_{T=t}]
    &= \EE_{T\sim\pi_\base(\cdot|X, U)}\left[\frac{\bbmone_{T=t}}{\pi_\base(T|X, U)}\right] \\
    &= \EE_{X, U}\left[
        \int_\RR \sum_\cT \left(\frac{\bbmone_{t'=t}}{\pi_\base(t'|X, U)}\right) p(y|t', X, U)\pi_\base(t'|X, U) \rd y \rd t' 
    \right] \\
    &= \EE_{X, U}\left[
        \sum_\cT \bbmone_{t'=t} \int_\RR p(y|t', X, U) \rd t' \rd y 
    \right]
    = 1,
\end{align*}
and the non-negativity condition is trivial from the definition of $w(y, t, x)$.
Combining the above with the relaxation of the conditional f-constraint, the following uncertainty set with the ZSB constraint can be defined:
\begin{equation}
    \cW_{f_{t, x}}^\ZSB := \left\{
    w(y, t, x) \geq 0:
    \begin{array}{c}
        \EE_{X, U} \left[f_{T, X}\left(p_\obs(T|X)w(Y, T, X)\right)\right] \leq \gamma \\
        \text{and} \\
        \EE_\obs[\bbmone_{T=t}w(Y, T, X)] = 1 \text{ for any } t\in\cT
    \end{array}
    \right\}.
\end{equation}
This uncertainty set has been traditionally adopted by many works such as \citet{tan2006distributional, zhao2019sensitivity, kallus2018confounding, kallus2021minimax}.
For this uncertainty set, associated lower bound
\begin{equation}
V_\textinf^\ZSB(\pi) := \inf_{w\in\cW^\ZSB}\EE_\obs\left[w(Y, T, X)\pi(T|X)Y \right]\\
\end{equation}
can be consistently approximated straightforwardly. 
By approximating the expectations by empirical average, we get a linear program with parameter $\bsw=(w_1, \ldots, w_n)^T = \left(w(Y_1, T_1, X_1), \ldots, w(Y_n, T_n, X_n)\right)^T$,
\begin{equation}
    \hat V_\textinf^\ZSB(\pi)
    := \min_{\bsw \in \hat\cW_{f_{t, x}}^\ZSB} \hat\EE_n [w(Y, T, X) \pi(T|X)Y ]
\label{eq:v_inf_zsb_empirical}
\end{equation}
where 
\begin{equation}
\hat\cW_{f_{t, x}}^\ZSB := \left\{
\bsw \geq \bszero:
\begin{array}{c}
    \hat\EE_n[f_{T, X}\left(p_\obs(T|X) w(Y, T, X)\right)] \leq \gamma \\
    \text{and}\\
    \hat\EE_n[\bbmone_{T=t}w(Y, T, X)] = 1 \text{ for all }t\in\cT
\end{array}\right\}.
\label{eq:zsb_empirical_uncertainty_set}
\end{equation}

Here we should note that \eqref{eq:v_inf_zsb_empirical} is not exactly the estimator used in \citet{zhao2019sensitivity} and its recent extensions. Instead, they solved $\min_{\bsw\geq\bszero}\frac{\hat \EE_n[w(Y, T, X)\pi(T|X)Y]}{\hat \EE_n[w(Y, T, X)]}$ such that $\hat \EE_n[f_{T, X}(p_\obs(T|X)w(Y, T, X))] \leq \gamma$, by the linear fractional programming.

\subsection{Conditional Moment Constraints}
Now, we introduce the sharper constraints that we leverage in our work.
We relax constraint \eqref{eq:proper_base_policy} as conditional moment constraints
\begin{equation}
    \EE_\obs[w(Y, T, X)|T=t, X=x] \cdot p_\obs(t|x) = 1 \text{ for any } t\in\cT \text{ and }x\in\cX
    \label{app-eq:conditional_moment_constraints}
\end{equation}
plus non-negativity constraint \eqref{app-eq:nonnegativity_constraint}.
For this relaxation, we do not require that action space $\cT$ is discrete and finite.
We can check the validity of this relaxation from 
\begin{align*}
\EE_\obs &[w(Y, T, X)|T=t, X=x]\cdot p_\obs(t, x) \\
&= \EE_\obs \left[
    \EE_{T\sim\pi_\base(\cdot|X, U)}\left[\frac{1}{\pi_\base(T|X, U)}|Y, T, X\right]
| T=t, X=x \right] \cdot p_\obs(t, x) \\
&= \EE_{T\sim\pi_\base(\cdot|X, U)}\left[ \frac{1}{\pi_\base(T|X, U)} | T=t, X=x \right] \cdot p_\obs(t, x) \\
&= \sumint_\cU \frac{1}{\pi_\base(t|x, u)} p_{\pi_\base}(u| t, x) \rd u \cdot p_\obs(t, x) \\
&= \sumint_\cU \frac{1}{\pi_\base(t|x, u)} p_{\pi_\base}(t, x, u) \rd u \\
&= \sumint_\cU p(x, u)\rd u 
= p_\obs(x)
\end{align*}
and the fact that $p_\obs(t|x) = p_\obs(t, x) / p_\obs(x)$.
Here, $p_{\pi_\base}(t, x, u)$ and $p_{\pi_\base}(u|t, x)$ denote the joint and conditional distribution of $T$, $X$, and $U$ under confounded contextual bandits model.
We again combine these conditional moment constraints (CMC) with the relaxed conditional f-constraint to obtain
\begin{equation}
    \cW_{f_{t, x}}^\CMC:= \left\{
    w(y, t, x) \geq 0:
    \begin{array}{c}
        \EE_{X, U} \left[f_{T, X}\left(p_\obs(T|X)w(Y, T, X)\right)\right] \leq \gamma \\
        \text{and} \\
        \EE_\obs[w(Y, T, X)|T=t, X=x] \cdot p_\obs(t|x) = 1 \text{ for any } t\in\cT \text{ and } x\in\cX
    \end{array}
    \right\}
\end{equation}
and its corresponding lower bound
\begin{equation}
V_\textinf^\CMC(\pi):= \inf_{w\in\cW^\CMC}\EE_\obs\left[w(Y, T, X)\pi(T|X)Y \right]. \label{app-eq:v_inf_conditional_moment_constraints}\\
\end{equation}

Naturally, for these uncertainty sets, one can show inclusion relations $\cW_{f_{t, x}} \subseteq \cW^\CMC_{f_{t, x}}\subseteq \cW^\ZSB_{f_{t, x}}$.
The former inclusion follows from the definition of $\cW^\CMC_{f_{t, x}}$. 
We can show the latter inclusion by taking the (conditional) expectation of conditional moment constraints \eqref{app-eq:conditional_moment_constraints} with respect to $p_\obs(x| t)$ to obtain ZSB constraint \eqref{app-eq:zsb_constraint}.

Here, we can also show that there exist cases where strict inclusion holds so that these sets are not equivalent.
For the latter inclusion, it is trivial to show that strict inclusion holds when the conditional moment constraints are stronger than the ZSB constraint.
For the former inclusion, we can construct the following toy example where there exists some $w(y, t, x) \in\cW^\CMC_{f_{t, x}} \setminus \cW_{f_{t, x}}$.

\begin{example}[A non-realizable element in $\cW^\CMC$]\label{ex:non_realizability}
Let us assume that the context space is a singleton and the action space and the reward space are binary so that $\cX=\{x\}$ and $\cT=\cY=\{-1, +1\}$.
Then, observational distribution 
\[
    p_\obs(Y=\pm 1, T=\pm 1, X=x) = 1/4 
\]
and inverse probability weight 
\[
    w(y, t, x) = \left\{
    \begin{array}{l}
        3.1 \text{\ \  if \ \ }y = -1 \\
        0.9 \text{\ \  if \ \ }y = +1
    \end{array}
    \right.
\]
satisfy conditional moment constraints \eqref{app-eq:conditional_moment_constraints} as well as non-negativity constraints \eqref{app-eq:nonnegativity_constraint} since
\[
    \EE_\obs[w(Y, T, X)|T=+1, X=x]\cdot p_\obs(T=+1|X=x) = \left(0.9 \cdot \frac{1}{2} + 3.1 \cdot \frac{1}{2}\right)\cdot \frac{1}{2} = 1
\]
and 
\[
    \EE_\obs[w(Y, T, X)|T=-1, X=x]\cdot p_\obs(T=-1|X=x) = \left(0.9 \cdot \frac{1}{2} + 3.1 \cdot \frac{1}{2}\right)\cdot \frac{1}{2} = 1.
\]
By considering some conditional f-constraint that contains the above parameter values, we can construct some uncertainty set $\cW^\CMC_{f_{t, x}}$ that contains this $w(y, t, x)$ given above observational distribution $p_\obs(y, t, x)$.
However, there exists no proper policy $\pi_\base(t|x, u)$ and underlying model $p(y|t, x, u)$ and $p(x, u)$ that satisfy both
\[
    p_\obs(y, t, x) = \sumint_\cU p(y|t, x, u)\pi_\base(t|x, u)p(x, u)\rd u
\]
and 
\begin{align*}
    w(y&, t, x) \cdot p_\obs(y, t, x) \\
    &= \EE_{T\sim\pi_\base(\cdot|X, U)}\left[\frac{1}{\pi_\base(T|X, U)}|Y=y, T=t, X=x\right] \cdot p_\obs(y, t, x)\\
    &= \sumint_\cU p(y|t, x, u)p(x, u)\rd u. 
\end{align*}
Indeed,
\begin{align*}
    p_\obs(Y=+1, t, x) 
    &= \sumint_\cU p(y|t, x, u)\pi_\base(t|x, u)p(x, u)\rd u \\
    &\leq \sumint_\cU p(y|t, x, u)p(x, u)\rd u  \\
    &= w(+1, t, x) \cdot p_\obs(Y=+1, t, x)
\end{align*}
contradicts our assumptions on the model parameters $p_\obs(Y=+1, t, x)=1/4$ and $w(+1, t, x)=0.9$.
In other words, such element $w(y, x, t)$ in uncertainty set $\cW_{f_{t, x}}^\CMC$ is not realizable, and thus, the strict inclusion holds for $\cW_{f_{t, x}} \subseteq \cW_{f_{t, x}}^\CMC$ in this case.
This example exploits the too much flexibility of the conditional f-constraint by taking the unrealizable case where $w(+1, t, x) < 1$. 
In the case of discrete action space, we know that $w\geq 1$ because the inverse probability of discrete action is always no less than $1$.
\end{example}

Having shown the inclusion relations of the uncertainty sets, we can discuss the relations among the lower bounds.
Assuming that the true base policy is contained in the original sensitivity model, we have $V \geq V_\textinf \geq V_\textinf^\CMC \geq V_\textinf^\ZSB$.
One important question to ask here is under what kind of constraints the second equality holds so that the lower bound of the conditional moment constraints is tight.
Surprisingly, recent work by \citet{dorn2022sharp} showed that equality $V_\textinf = V_\textinf^\CMC$ holds for average treatment effect estimation with Tan's marginal sensitivity model \citep{tan2006distributional}.
They showed that there exists minimizer $w^*(y, t, x)$ of $\mathrm{ATE}(w):=\EE_\obs[w(Y, T, X)\bbmone_{T=1}Y] - \EE_\obs[w(Y, T, X)\bbmone_{T=0}Y]$, which is realizable so that there exists $p(y|t, x, u)$, $\pi_\base(t|x, u)$, and $p(x, u)$ that is compatible with any $p_\obs(y, t, x)$ and minimizer $w^*(y, t, x)$.
However, they did not provide any realizability results for more general settings such as the evaluation of general policy, non-binary action spaces, and general box and f-divergence constraints.
Clearly, Example \ref{ex:non_realizability} shows that the same equality does not always hold for any box constraints, as it can be the minimizer for some box constraints and some policies.
However, the question of which constraint class yields a tight bound with a realizable minimizer is an open question.

\section{Alternative f-sensitivity Models by Jin et al. (2022)}

Here, we describe the difference between our f-sensitivity models and the f-sensitivity models proposed by \citet{jin2022sensitivity} that also uses uncertainty sets defined with the f-divergence.
They proposed a similar uncertainty set that relaxes the condition
\[
    D_f\left[p\left(Y(1)|X, T=1\right) || p\left(Y(1)|X, T=0\right)\right] \leq \gamma,
\]
where variable $Y(1)$ is the potential outcome variable for treatment $T=1$ in Rubin's potential outcome framework \citep{rubin2005causal} with binary treatment.
Under the assumption of unconfoundedness, the potential outcome variable $Y(1)$ must satisfy $Y(1) \indep T|X$, and thus it must satisfy 
$D_f\left[p\left(Y(1)|X, T=1\right) || p\left(Y(1)|X, T=0\right)\right]=0$
almost surely with respect to $p_\obs$.
Their sensitivity model can be interpreted as a relaxation of this assumption by allowing the violation of it up to $\gamma$.

In terms of the modeling paradigm, our f-sensitivity model follows the same modeling framework as Tan (2006), which takes into account the difference between observational policy $p_\obs(t|x)$ and underlying confounded policy $\pi_\base(t|x, u)$.
On the other hand, the model by \citep{jin2022sensitivity} considers the distributional shift between observation $Y(1)|X = x, T = 1$ and counterfactual $Y(1)|X = x, T = 0$, and therefore, their way of modeling is different from the one that Tan (2006) and its extension is based on.

\section{Quantile Balancing Estimator by Dorn and Guo (2022)}

In this section, we discuss the recently proposed tractable estimators for lower bound $V_\textinf^\CMC$ by \citet{dorn2022sharp}.
As the optimization for $V_\textinf^\CMC$ involves infinite dimensional constraints for all $t\in\cT$ and $x\in\cX$, we cannot apply the method analogous to \eqref{eq:v_inf_zsb_empirical} to obtain the empirical version of the lower bound.
They studied box constraints $a_w(t, x) \leq w(y, t, x) \leq b_w(t, x)$ in the case of discrete finite action space and proposed the first tractable method to impose such constraints.
They showed that we can solve \eqref{app-eq:v_inf_conditional_moment_constraints} in the case of box constraints as
\[
    V_{\textinf, \ \rbox}^\CMC =\min_{a_w(t, x) \leq w(y, t, x) \leq b_w(t, x)}\EE_\obs[w(Y, T, X)\pi(T|X)Y]
\]
subject to
\begin{equation}
    \EE_\obs[w(Y, T, X) \pi(T|X) Q(T, X)] = \EE_\obs\left[\left(\frac{\pi(T|X)}{p_\obs(T|X)}\right)Q(T, X)\right],
    \label{eq:quantile_balancing_constraints}
\end{equation}
where $Q(t, x)$ denotes the $\tau(t, x)$-quantile of the conditional distribution of $Y$ given $T=t$ and $X=x$ for $\tau(t, x):=\frac{1 / p_\obs(t|x) - a_w(t, x)}{b_w(t, x) - a_w(t, x)}$.
In the case of the marginal sensitivity model by \citet{tan2006distributional}, the expression for $\tau(t, x)$ can be simplified as $\tau(x) = \frac{1}{1 + \Gamma}$.
They leveraged the above characterization for $V_\textinf^\CMC$ by plugging in estimate $\hat Q(t, x)$ of the conditional quantile function and empirically approximating the expectation to obtain a tractable linear programming problem for the quantile balancing (QB) estimator,
\[
    \hat V_{\textinf,\ \rbox}^\text{QB} =\min_{a_w(t, x) \leq w_i \leq b_w(t, x)} \hat\EE_n[w(Y, T, X)\pi(T|X)Y]
\]
subject to
\[
    \hat\EE_n[w(Y, T, X)\pi(X|T)\hat Q(T, X)] = \hat\EE_n \left[\left(\frac{\pi(T|X)}{p_\obs(T|X)}\right)\hat Q(T, X)\right].
\]

Indeed, we will show later that this quantile balancing estimator is a special case of our estimator where we have a nearly optimal choice of orthogonal function class $\{\psi_1\}$ where $\psi_1(x, t) = \left(\frac{\pi(t|x)}{p_\obs(t|x)}\right)\hat Q(t, x)$.
Therefore, the KCMC estimator is guaranteed to be no looser than QB estimator when $\left(\frac{\pi(t|x)}{p_\obs(t|x)}\right)\hat Q(t, x) \in \mathrm{span}\left(\{\psi_d(t, x)\}_{d=1}^D\right)$.
To put it in another way, if we estimate the quantile by linear quantile regression with feature vectors $\left\{ \left( \frac{p_\obs(t|x)}{\pi(t|x)} \right) \psi_d(x, t)\right\}_{d=1}^D$, the QB estimator is no tighter than the KCMC estimator.

As our estimator generalizes the previous work, our estimator overcomes some drawbacks of the quantile balancing estimators.
For example, the quantile balancing estimator cannot handle policy learning and the f-divergence constraint.
Policy learning is difficult with the quantile balancing estimator because taking the derivative with respect to policy requires differentiability of the solution of the above linear programming with respect to parameter $\hat Q$.  
Moreover, the quantile balancing method is designed only for box constraints and does not have a proper extension to the f-sensitivity model.
In contrast, our estimator of sharper bound $V_\textinf^\CMC$ based on the kernel method can naturally handle the above-mentioned generalized cases of sensitivity analysis.

\section{More Detail on Low-rank Gaussian Process Constraints} 

Here, we provide more discussion on the derivation and interpretations of the low-rank Gaussian process kernel conditional moment constraints.

\subsection{Derivation}

Let's consider spectral decomposition of the kernel matrix $K=V\Lambda V^T$ with orthonormal matrix $V\in\RR^{n\times n}$ and diagonal matrix $\Lambda = \mathrm{diag}(\lambda_1, \ldots, \lambda_n)\in \RR^{n\times n}$ satisfying $\lambda_1 \geq \lambda_2 \geq \ldots \geq \lambda_n \geq 0$.
By truncating the spectrum after the first $D$ dimensions, we approximate the kernel matrix as $K\approx \tilde K := \tilde V \tilde \Lambda {\tilde V}^T$, where $\tilde \Lambda = \mathrm{diag}(\lambda_1, \ldots, \lambda_D)\in \RR^{D\times D}$ and $\tilde V \in \RR^{n\times D}$ is the first $D$ columns of matrix $V$.
Then, we can obtain the low-rank version of the Gaussian process regression model as 
\begin{equation}
\begin{array}{l}
    \bsz \sim \mathcal{N}(\bszero, \tilde \Lambda) \\
    \bsg = \tilde V \bsz, \\
    \bse = \bsg + \varepsilon, \\
    \varepsilon \sim\mathcal{N}(\bszero, \sigma^2I_n).
\end{array}
\label{eq:low_rank_gaussian_process_regression_model}
\end{equation}
Here, we can verify that this model is a valid low-rank approximation, by checking that the prior mean and covariance of $\bsg$ are $\bszero$ and $\tilde K$ respectively. 

For this low-rank model, the posterior distribution of $\bsg$ given $\bse$ is again a multivariate normal distribution with mean and variance $\tilde\mu_{\bsg|\bse} = \tilde K( \tilde K + \sigma^2 I_n)^{-1} \bse$ and $\tilde\Sigma_{\bsg|\bse} = \tilde K - \tilde K(\tilde K + \sigma^2 \mathrm{I})^{-1}\tilde K$.
Therefore, we can analogously define the credible set and the kernel conditional moment constraints as
\begin{equation}
\tilde C_{\bsg|\bse} = \{\bsg : (\bsg - \tilde\mu_{\bsg|\bse})^T \tilde\Sigma_{\bsg|\bse}^{-1} (\bsg - \tilde\mu_{\bsg|\bse}) \leq \chi^2_{D}(1 - \alpha)\}.
\end{equation}
and
\begin{align}
\begin{split}
    \KCMC_\text{low-rank GP}
    &\stackrel{\text{def}}{\Leftrightarrow} \bszero \in \tilde C_{\bsg|\bse} \\
    &\Leftrightarrow (\bsw \odot \bsp_{\obs, T|X} - 1)^T M_\text{low-rank GP} (\bsw \odot \bsp_{\obs, T|X} - 1) \leq \chi^2_{D}(1 - \alpha)
\end{split}
\end{align}
where $M_\text{low-rank GP}:=(\tilde K+\sigma^2I_n)^{-1}\tilde K(\tilde K - \tilde K(\tilde K+\sigma^2I_n)^{-1}\tilde K)^{-1}\tilde K(\tilde K+\sigma^2I_n)^{-1}$.

A big difference between the credible set of the original Gaussian process and the low-rank Gaussian process is the degree of freedom of the $\chi^2$ distribution.
This is because the distribution of the $\bsg$ is essentially restricted to the $D$-dimensional subspace spanned by columns of $\tilde V$.
The condition that the credible set of $\bsg$ includes $\bszero\in\RR^n$ is equivalent to the condition that $\bszero\in\RR^D$ is contained in the credible set of $\bsz$, whose posterior is $D$-dimensional multivariate normal distribution.

\subsection{A Spectral Interpretation of Low-rank and Full-rank Constraints}
With the spectral decomposition, we can obtain more intuitive expressions for quadratic forms $\bse^T M_\text{GP} \bse$ and $\bse^T M_\text{low-rank GP}\bse$, which are
\begin{align}
\bse^T M_\text{GP} \bse &= \frac{1}{\sigma^2}(V^T\bse)^T\mathrm{diag}\left(\frac{\lambda_1}{\lambda_1 + \sigma^2}, \ldots, \frac{\lambda_n}{\lambda_n + \sigma^2}\right)(V^T\bse), \\
\bse^T M_\text{low-rank GP} \bse &= \frac{1}{\sigma^2}(\tilde V^T\bse)^T\mathrm{diag}\left(\frac{\lambda_1}{\lambda_1 + \sigma^2}, \ldots, \frac{\lambda_D}{\lambda_D + \sigma^2}\right)(\tilde V^T\bse).
\end{align}
Now we can interpret the above quantities from the perspective of the kernel principal component analysis (KPCA) \citep{scholkopf1997kernel}.
Let $\phi^\KPCA_1, \ldots, \phi^\KPCA_n$ be the empirically normalized principal components in RKHS corresponding to the empirical spectrum $\lambda_1/n, \ldots, \lambda_n/n$, which satisfy $\hat \EE_n\left|\phi^\KPCA_i(Y, T, X)\right|^2 = 1$ for any $i$.
Then, it can be shown that $\phi^\KPCA_d(T_i, X_i) = \sqrt{n} V_{d, i}$, and therefore, 
$V^T\bse / \sqrt{n} = \left(\hat\EE_n[\phi^\KPCA_1 e], \ldots, \hat\EE_n[\phi^\KPCA_n e] \right)^T$ 
and $\tilde V^T\bse / \sqrt{n} = \left(\hat\EE_n[\phi^\KPCA_1 e], \ldots, \hat\EE_n[\phi^\KPCA_D e] \right)^T$.
Using these relations, we can re-write the above quadratic forms as
\begin{align}
\begin{split}
\bse^T M_\text{GP} \bse &= \frac{1}{\sigma^2}\sum_{d=1}^{n}\frac{\lambda_d / n}{\lambda_d / n + \sigma^2 / n}\left|\sqrt{n}\hat\EE_n[\phi^\KPCA_d(T, X)e(Y, T, X)]\right|^2 \\
\bse^T M_\text{low-rank GP} \bse &= \frac{1}{\sigma^2}\sum_{d=1}^{D}\frac{\lambda_d / n}{\lambda_d / n + \sigma^2 / n}\left|\sqrt{n}\hat\EE_n[\phi^\KPCA_d(T, X)e(Y, T, X)]\right|^2.
\end{split}
\label{eq:quadratic_constraint_spectral_decomposition}
\end{align}
These expressions provide several intuitions of the kernel conditional moment constraints. 
First, the difference in the degree of freedom for the $\chi^2$ upper bounds clearly corresponds to the number of summands for individual constraints.
Second, when we decompose the individual summand into weight $\frac{\lambda_d / n}{\lambda_d / n + \sigma^2 / n}$ and squared penalty $\left|\sqrt{n}\hat\EE_n[\phi^\KPCA_d(T, X)e(Y, T, X)]\right|^2$, the weight term discounts the squared penalty as the magnitude of the spectrum $\lambda_d$ decays.
Especially, the summand becomes negligible when $\lambda_d$ is significantly smaller than $\sigma^2$.
If $\lambda_D$ is small enough compared to $\sigma^2$, both quadratic forms are approximately equal, and thus, the low-rank constraints can be tighter than its full-rank counterpart owing to the smaller $\chi^2$ upper bound.
This is a side benefit of low-rank constraints that can also justify its use in practice.
Third, the squared penalty term can be interpreted as a soft version of constraint $\EE_\obs[e(Y, T, X)\phi^\KPCA_d(T, X)]=0$.
Lastly, we can approximately apply the central limit theorem to $\sqrt{n}\hat\EE_n[\phi^\KPCA_d(T, X)e(Y, T, X)]$ and obtain its asymptotic distribution. 
For the sake of approximation, let us assume that $\VV_\obs[e(Y, T, X)|T, X] = \EE_\obs[|e(Y, T, X)|^2|T, X] = \sigma^2$ and that $\phi^\KPCA_d(t, x)$ is fixed for any $n$ and it satisfies $\EE_\obs|\phi^\KPCA_d|^2\approx1$.
Then, $\sqrt{n}\hat\EE_n[\phi^\KPCA_d(T, X)e(Y, T, X)]$ is asymptotically normal with mean zero and variance 
$\VV_\obs[\phi^\KPCA_d e] =$
$\EE_\obs|\phi^\KPCA_d e|^2=$
$\EE_\obs\left[|\phi^\KPCA_d(T, X)|^2 \EE_\obs[|e(Y, T, X)|^2|T, X]\right] =$
$\sigma^2\EE_\obs|\phi^\KPCA(T, X)|^2 $
$\approx \sigma^2$.
Therefore, we see that the above quadratic forms are the discounted sum of squares of asymptotically standard normal random variables, which again gives an intuition for the $\chi^2$ upper bound.

\subsection{Choice of Parameter \texorpdfstring{$\sigma^2 > 0$}{sigma}}
Finally, we discuss a practice consideration on the choice of $\sigma^2$.
From \eqref{eq:quadratic_constraint_spectral_decomposition}, we can see that the quadratic constraint gets tighter when we reduce the value of $\sigma^2$, as scaling factor $\frac{1}{\sigma^2}$ of the summation, as well as discounting weight $\frac{\lambda_d / n}{\lambda_d / n + \sigma^2 / n}$, gets larger as we decrease $\sigma^2$. 
In the sensitivity analysis, we are mostly interested in extreme cases, and thus, it is better to have overly loose constraints than overly tight constraints.
Therefore, when choosing $\sigma^2$, we can take the largest possible value of $\sup_{t\in\cT, x\in\cX}\VV_\obs[e|T=t, X=x]$.
In the case of the box constraints, we can calculate the upper and lower bound of $e(y, t, x)$, and therefore, the difference between both bounds can be a good guess of the supremum.
With regard to the f-sensitivity model, there is no way to obtain such a bound.
For this issue, a practical workaround may be to define the uncertainty set by combining both the f-divergence constraints and the box constraints, so that we can still use the choice of $\sigma^2$ for the box constraints.

\section{Theoretical Analysis}

In this section, we study the property of the kernel conditional moment constraints in the confounding robust inference.
For the convenience of the analysis, we only consider low-rank orthogonality condition $\hat \EE_n[e(Y, T, X)\bspsi(T, X)] = \bszero$ as the kernel conditional moment constraints in this section.
Thus, we focus on the property of the following population lower bound 
\begin{equation}
    V_\textinf^\KCMC(\pi) = \inf_{w\in\cW^\KCMC_{f_{t, x}}}\EE_\obs[w(Y, T, X)\pi(T|X)Y]
    \label{app-eq:v_inf_kernel_conditional_moment_constraints}
\end{equation}
for
\begin{equation}
    \cW^\KCMC_{f_{t, x}} = \left\{
    w(y, t, x)\geq 0:
    \begin{array}{c}
         \EE_\obs[f_{T, X}(w(Y, T, X)p_\obs(T|X))] \leq \gamma \\
         \text{and} \\
         \EE_\obs[\left(w(Y, T, X)p_\obs(T|X) - 1\right)\psi_d] = 0 \text{ \ for any }d = 1, \ldots, D
    \end{array}
    \right\}
\end{equation}
and its empirical version 
\begin{equation}
    \hat V_\textinf^\KCMC(\pi) = \inf_{\bsw\in\hat\cW^\KCMC_{f_{t, x}}}\hat\EE_n[w(Y, T, X)\pi(T|X)Y]
    \label{app-eq:empirical_v_inf_kernel_conditional_moment_constraints}
\end{equation}
for
\begin{equation}
    \hat\cW^\KCMC_{f_{t, x}} = \left\{
    w(y, t, x) \geq 0:
    \begin{array}{c}
         \hat\EE_n[f_{T, X}(w(Y, T, X)p_\obs(T|X))] \leq \gamma \\
         \text{and} \\
         \hat\EE_n[\left(w(Y, T, X)p_\obs(T|X) - 1\right)\psi_d] = 0 \text{ \ for any }d = 1, \ldots, D
    \end{array}
    \right\}.
\end{equation}
We first study the property of the minimizers for the above problems.
Then, we analyze the specification error of the kernel conditional moment constraints $\left|V^\CMC_\textinf - V^\KCMC_\textinf\right|$. 
We then provide a condition on orthogonal function set $\{\psi_d(t, x)\}_{d=1}^D$ under which the specification error becomes zero.
Finally, we study empirical estimator $\hat V_\textinf^\KCMC$ and prove consistency guarantees for both policy evaluation and learning.

Before further discussion, we introduce several simplifications of notations.
We omit subscripts of $\cW_{f_{t, x}}^\CMC$, $\cW_{f_{t, x}}^\KCMC$, $\hat\cW_{f_{t, x}}^\KCMC$ and $\EE_\obs$, unless they are unclear from the context.
We also introduce $r(y, t, x) = \left(\frac{\pi(t|x)}{p_\obs(t|x)}\right)\cdot y$ and re-parametrization $\tilde w(y, t, x) = p_\obs(t|x)w(y, t, x)$.

Furthermore, we introduce the subgradient and the Fenchel conjugate here, as we will make heavy use of them in this section.
The subgradient of convex function $f$ is represented by $\partial f$.
When we apply the addition operator to the subgradient, it represents the Minkowski sum.
Other operations to the subgradient such as multiplication are similarly defined.
The Fenchel conjugate of $f:\RR\to\RR$ is defined as $f^*(v):=\sup_{u}\{uv - f(u)\}$.
There are a few important properties of the Fenchel conjugate.
The Fenchel conjugate is always convex because the supremum of a family of convex functions is convex.
\footnote{In this case, $v\mapsto vu - f(u)$ is linear, and therefore, is convex.}
Additionally, there exists maximizer $u^*$ that solves $f^*(v) = \max_{u}\left\{vu - f(u)\right\}$ and it satisfies 
\begin{equation}
    u^* \in \partial f^*(v)
    \label{eq:fenchel_conjugate_solution}
\end{equation}
if $f$ is closed and convex.
Function $f$ is closed and convex if its epigraph $\mathrm{epi}(f):=\{(u, t): u\in\mathrm{dom}(f), \ f(u)\leq t\}$ is closed and convex, and these conditions are satisfied in our problem.
A more thorough treatment of the subgradient and the Fenchel conjugate can be found in  \citep{boyd2004convex}.

\subsection{Characterization of Solutions}

(The derivation of the dual problem below has errors, but they have fairly straightforward fix as discussed in Appendix \ref{app:errata}!)

In this section, we derive explicit formulae for the minimizers that give three lower bounds $V_\textinf^\CMC$, $V_\textinf^\KCMC$, and $\hat V_\textinf^\KCMC$, which are
\begin{align}
\begin{split}
    w^*_\CMC &= \arg\min_{w\in\cW^\CMC}\EE[w(Y, T, X)\pi(T|X)Y], \\
    w^*_\KCMC &= \arg\min_{w\in\cW^\KCMC}\EE[w(Y, T, X)\pi(T|X)Y], \\
    \hat w_\KCMC &= \arg\min_{w\in\hat\cW^\KCMC}\hat\EE_n[w(Y, T, X)\pi(T|X)Y].
\end{split}
\label{eq:w_solutions}
\end{align}
Here, we know that these problems have the minimizers because the above problems are minimizations of linear objectives under convex constraints.
Furthermore, we know that the strong duality holds for the above convex optimizations, as their feasibility sets have a non-empty relative interior, satisfying Slater's constraint qualification.
Using these properties, we obtain the following lemma:

\begin{lemma}[Characterization of solutions]\label{lemma:w_characterization}
Let $w^*_\CMC$, $w^*_\KCMC$, and $\hat w_\KCMC$ be defined as in \eqref{eq:w_solutions}. 
Then, there exist function $\eta_\CMC:\cT\times\cX\to\RR$, vectors $\eta_\KCMC, \eta_\KCMC'\in\RR^D$, and constants $\eta_f, \eta_f', \eta_f'' > 0$ such that 
\begin{align}
w^*_\CMC(y, t, x) &\in \left( \frac{1}{p_\obs(t|x)} \right) \partial f_{t, x}^* \left( \frac{\eta_\CMC(t, x) - r(y, t, x)}{\eta_f} \right),
\label{eq:cmc_solution_characterization} \\
w^*_\KCMC(y, t, x) &\in \left( \frac{1}{p_\obs(t|x)} \right) \partial f_{t, x}^* \left( \frac{{\eta_\KCMC}^T\bspsi(t, x) - r(y, t, x)}{\eta_f'} \right),
\label{eq:kcmc_solution_characterization} \\
\hat w_\KCMC(y, t, x) &\in \left( \frac{1}{p_\obs(t|x)} \right) \partial f_{t, x}^* \left( \frac{{\eta_\KCMC'}^T\bspsi(t, x) - r(y, t, x)}{\eta_f''} \right).
\label{eq:empirical_kcmc_solution_characterization}
\end{align}  
\end{lemma}

\begin{proof}
    See below.
\end{proof}

\subsubsection{Characterization of \texorpdfstring{$w^*_\CMC$}{the solution for conditional moment constraint}}
By using the strong duality, we can transform the original problem for $V_\textinf^\CMC$ as
\begin{align}
V_\textinf^\CMC
&= \inf_{w\in\cW^\CMC}\EE[\tilde w(Y, T, X)r(Y, T, X)] \\
&= \inf_{\tilde w}\sup_{\substack{\eta_\CMC:\cT\times\cX\to\RR,\\ \eta_f\geq 0}}
    \EE[\tilde w r] 
    - \EE\left[\eta_\CMC(T, X)\EE[\tilde w - 1|T, X]\right] 
    + \eta_f\left(\EE[f_{T, X}(\tilde w)] - \gamma\right) \\
&= \sup_{\substack{\eta_\CMC:\cT\times\cX\to\RR,\\ \eta_f\geq 0}}\inf_{\tilde w}
    \EE[\tilde w r] 
    - \EE\left[\eta_\CMC(T, X)(\tilde w - 1)\right] 
    + \eta_f\left(\EE[f_{T, X}(\tilde w)] - \gamma\right) \\
&= \sup_{\substack{\eta_\CMC:\cT\times\cX\to\RR,\\ \eta_f > 0}}
    - \eta_f \gamma 
    + \EE[\eta_\CMC]
    - \eta_f \EE\left[\sup_{\tilde w}\left\{\left(\frac{\eta_\CMC - r}{\eta_f}\right)\tilde w - f_{T, X}(\tilde w)\right\}\right]  \\
&= \sup_{\substack{\eta_\CMC:\cT\times\cX\to\RR,\\ \eta_f > 0}}
    - \eta_f \gamma 
    + \EE[\eta_\CMC]
    - \eta_f \EE\left[f_{T, X}^*\left(\frac{\eta_\CMC - r}{\eta_f}\right)\right].
    \label{app-eq:v_inf_conditional_moment_constraints_dual}
\end{align}
In the second last line, we assumed $\eta_f > 0$\st{, as the inner minimization achieves $-\infty$ if we take $\eta_f = 0$ (unless $r(Y, T ,X)$ does not depend on $Y$ almost surely)} (See the list of errata in Appendix \ref{app:errata}).

Now, as primal solution $w^*_\CMC$ must satisfy the Karush–Kuhn–Tucker (KKT) conditions, we can take the maximizers of \eqref{app-eq:v_inf_conditional_moment_constraints_dual} as $\eta_f^*$ and $\eta_\CMC^*(t, x)$.
Using property \eqref{eq:fenchel_conjugate_solution} of the Fenchel conjugate, we can obtain solution form of $w^*_\CMC$
\begin{equation}
    w^*_\CMC(y, t, x) \in  \left( \frac{1}{p_\obs(t|x)} \right) \partial f_{t, x}^* \left( \frac{\eta_\CMC^*(t, x) - r(y, t, x)}{\eta_f^*} \right),
\end{equation}
which proves \eqref{eq:cmc_solution_characterization}. Here, factor $1/p_\obs(t|x)$ in front of the subgradient appears because of reparametrization $\tilde w(y, t, x) = p_\obs(t|x)w(y, t, x)$.

Now we are interested in dual solutions $\eta_f^*$ and $\eta_\CMC^*$.
To characterize the dual solution, we take the stationary conditions for the dual problem as
\begin{align}
0 &\in \partial_{\eta_f} \left(
    - \eta_f \gamma 
    + \EE[\eta_\CMC]
    - \eta_f \EE\left[f_{T, X}^*\left(\frac{\eta_\CMC - r}{\eta_f}\right)\right]
\right) \\
&= - \gamma 
    - \EE\left[f_{T, X}^*\left(\frac{\eta_\CMC - r}{\eta_f}\right)\right]
    + \EE\left[\left(\frac{\eta_\CMC - r}{\eta_f}\right) \cdot \partial f_{T, X}^*\left(\frac{\eta_\CMC - r}{\eta_f}\right)\right] 
    \label{eq:eta_f_characterization_conditional_moment_constraints}
\end{align}
and
\begin{align}
0 &\in \partial_{\eta_\CMC} \left(
    - \eta_f \gamma 
    + \EE[\eta_\CMC]
    - \eta_f \EE\left[f_{T, X}^*\left(\frac{\eta_\CMC - r}{\eta_f}\right)\right]
\right) \\
&= 1 - \EE\left[\partial f_{T, X}^*\left(\frac{\eta_\CMC - r}{\eta_f}\right)|T=t, X=x\right],
    \label{eq:eta_cmc_characterization_conditional_moment_constraints}
\end{align}
where we used the functional gradient on measure $p_\obs$ in the second subgradient condition. 
As $\tilde w^*_\CMC(y, t, x) := p_\obs(t|x) \cdot w^*_\CMC(y, t, x) \in \partial f_{t, x}^* \left( \frac{\eta_\CMC^*(t, x) - r(y, t, x)}{\eta_f^*} \right)$, we can see that the second condition corresponds to the conditional moment constraints.

In general cases, it is difficult to derive analytical expressions for solutions $\eta_f^*$ and $\eta^*_\CMC$, as well as $w^*_\CMC$.
However, we can actually obtain their explicit expressions in the case of box constraints.

\begin{example}[Solutions for box constraints]\label{ex:box_constraint_analytical_solution}
Let us consider the box constraints in the form of the conditional f-constraint.
Here, for notational simplicity, we omit the subscript of $a_{\tilde w}$ and $b_{\tilde w}$.
For example, we will simply write $f_{t, x}(\tilde w) = I_{[a(t, x), b(t, x)]}(\tilde w)$.
Then, for this choice of function $f_{t, x}$, we can derive its conjugate and its subgradient as
\begin{equation}
    f_{t, x}^*(v) = \begin{cases}
        a(t, x)v & \text{ if \ \ }v < 0, \\
        0 & \text{ if \ \ }v = 0, \\
        b(t, x)v & \text{ if \ \ }v > 0,
    \end{cases}
\end{equation}
and 
\begin{equation}
    \partial f_{t, x}^*(v) = \begin{cases}
        a(t, x) & \text{ if \ \ }v < 0, \\
        [a(t, x), b(t, x)] & \text{ if \ \ }v = 0, \\
        b(t, x) & \text{ if \ \ }v > 0.
    \end{cases}
\end{equation}
Substituting the above expression of $\partial f_{t, x}^*$ in the subgradient term of condition \eqref{eq:eta_cmc_characterization_conditional_moment_constraints}, we can derive more explicit expression
\begin{align}
&\EE\left[\partial f_{T, X}^*\left(\frac{\eta_\CMC - r}{\eta_f}\right)|T=t, X=x\right] \\
&= \PP(r < \eta_\CMC(t, x)) \cdot b(t, x)
        + \PP(r = \eta_\CMC(t, x)) \cdot [a(t, x), b(t, x)]
        + \PP(r > \eta_\CMC(t, x)) \cdot a(t, x).
\end{align}
Here, we used the box constraints' property $a(t, x) \leq 1 \leq b(t, x)$, which follows from the requirement that $f_{t, x}$ must satisfy $f_{t, x}(1)=0$ for any $t\in\cT$ and $x\in\cX$.
Assuming that the conditional distribution of $r(Y, T, X)$ given $T=t$ and $X=x$ yields continuous distribution for any $t\in\cT$ and $x\in\cX$, the second subgradient condition becomes
\begin{equation}
1 = \PP(r \leq \eta_\CMC(t, x)) \cdot b(t, x)
        + \PP(r > \eta_\CMC(t, x)) \cdot a(t, x).
\end{equation}
This implies that $\PP(r \leq \eta_\CMC(t, x)) = \frac{1 - a(t, x)}{b(t, x) - a(t, x)}=: \tau(t, x)$, and therefore,
\[
    \eta^*_\CMC(t, x) = \left( \frac{\pi(t|x)}{p_\obs(t|x)} \right)Q(t, x)
\] 
where $Q(t, x)$ was defined as the $\tau(t, x)$-th quantile of the conditional distribution of $Y$ given $T=t$ and $X=x$.
From this dual solution, the primal solution can also be recovered using \eqref{eq:cmc_solution_characterization} as
\begin{equation}
    w^*_\CMC(y, t, x) =
    \begin{cases}
        b(t, x) & \text{ if \ \ }y \leq Q(t, x), \\
        a(t, x) & \text{ otherwise}.
    \end{cases}
\end{equation}
\end{example}

\subsubsection{Characterization of \texorpdfstring{$w^*_\KCMC$}{the solution for the population kernel conditional moment constraints}}
Now, we derive the characterization of $w^*_\KCMC$.
Let $\bspsi(t, x) = \left(\psi_1(T, X), \ldots, \psi_D(T, X)\right)^T$.
We can use exactly the same technique as above and reach a similar characterization of the solution as
\begin{align}
V_\textinf^\KCMC
&= \inf_{w\in\cW^\KCMC}\EE[\tilde w(Y, T, X)r(Y, T, X)] \\
&= \inf_{\tilde w}\sup_{\substack{\eta_\KCMC\in\RR^D,\\ \eta_f\geq 0}}
    \EE[\tilde w r] 
    - \EE\left[(\tilde w - 1) {\eta_\KCMC}^T\bspsi\right] 
    + \eta_f\left(\EE[f_{T, X}(\tilde w)] - \gamma\right) \\
&= \sup_{\substack{\eta_\KCMC\in\RR^D,\\ \eta_f\geq 0}}
    \EE[\tilde w r] 
    - \EE\left[(\tilde w - 1) {\eta_\KCMC}^T\bspsi\right] 
    + \eta_f\left(\EE[f_{T, X}(\tilde w)] - \gamma\right) \\
&= \sup_{\substack{\eta_\KCMC\in\RR^D,\\ \eta_f > 0}}
    - \eta_f \gamma 
    + {\eta_\KCMC}^T\EE[\bspsi]
    - \eta_f \EE\left[\sup_{\tilde w}\left\{\left(\frac{{\eta_\KCMC}^T\bspsi - r}{\eta_f}\right)\tilde w - f_{T, X}(\tilde w)\right\}\right]  \\
&= \sup_{\substack{\eta_\KCMC\in\RR^D,\\ \eta_f > 0}}
    - \eta_f \gamma 
    + {\eta_\KCMC}^T\EE[\bspsi]
    - \eta_f \EE\left[f_{T, X}^*\left( \frac{{\eta_\KCMC}^T\bspsi - r}{\eta_f}\right)\right].
    \label{app-eq:v_inf_kernel_conditional_moment_constraints_dual}
\end{align}
Now, using the maximizers of dual problem \eqref{app-eq:v_inf_kernel_conditional_moment_constraints_dual} $\eta_f^*$ and $\eta_\KCMC^*$, we can obtain a characterization of $w^*_\KCMC$ as
\begin{equation}
    w^*_\KCMC(y, t, x) \in  \left( \frac{1}{p_\obs(t|x)} \right) \partial f^* \left( \frac{{\eta_\KCMC^*}^T\bspsi(t, x) - r(y, t, x)}{\eta_f^*} \right),
    \label{eq:w_solution_kernel_conditional_moment_constraints}
\end{equation}
which proves \eqref{eq:kcmc_solution_characterization}.

Now we study the characterization of dual solutions $\eta_f^*$ and $\eta_\CMC^*$.
Again, we take the stationary conditions as
\begin{align}
0 &\in \partial_{\eta_f} \left(
    - \eta_f \gamma 
    + {\eta_\KCMC}^T\EE[\bspsi]
    - \eta_f \EE\left[f_{T, X}^*\left(\frac{{\eta_\KCMC}^T\bspsi - r}{\eta_f}\right)\right]
\right) \\
&= - \gamma 
    - \EE\left[f_{T, X}^*\left(\frac{{\eta_\KCMC}^T\bspsi - r}{\eta_f}\right)\right]
    + \EE\left[\left(\frac{{\eta_\KCMC}^T\bspsi - r}{\eta_f}\right) \cdot \partial f_{T, X}^*\left(\frac{{\eta_\KCMC}^T\bspsi - r}{\eta_f}\right)\right] 
    \label{eq:eta_f_characterization_kernel_conditional_moment_constraints}
\end{align}
and
\begin{align}
0 &\in \partial_{\eta_\KCMC} \left(
    - \eta_f \gamma 
    + {\eta_\KCMC}^T\EE[\bspsi]
    - \eta_f \EE\left[f_{T, X}^*\left(\frac{{\eta_\KCMC}^T\bspsi - r}{\eta_f}\right)\right]
\right) \\
&= \EE\left[\bspsi \left(
    1 - \partial f_{T, X}^*\left(\frac{{\eta_\KCMC}^T\bspsi - r}{\eta_f}\right)
\right)|T=t, X=x\right].
\label{eq:eta_kcmc_characterization_kernel_conditional_moment_constraints}
\end{align}
Again, as $\tilde w^*_\KCMC(y, t, x) = p_\obs(t|x) \cdot w^*_\KCMC(y, t, x) \in \partial f_{t, x}^* \left( \frac{{\eta_\KCMC^*}^T\bspsi(t, x) - r(y, t, x)}{\eta_f^*} \right)$, we can see that the second condition corresponds to the kernel conditional moment constraints.

\subsubsection{Characterization of \texorpdfstring{$\hat w_\KCMC$}{the solution for the kernel conditional moment constraints}}
Again, using the same techniques, we can derive the characterization of $\hat w_\KCMC$.
By exchanging $\EE$ with $\hat \EE_n$ in the proof for $w^*_\KCMC$ above and writing the maximizers of dual problem
\begin{equation}
    \sup_{\substack{\eta_\KCMC\in\RR^D,\\ \eta_f > 0}}
    - \eta_f \gamma 
    + {\eta_\KCMC}^T\hat\EE_n[\bspsi]
    - \eta_f \hat\EE_n\left[f_{T, X}^*\left( \frac{{\eta_\KCMC}^T\bspsi - r}{\eta_f}\right)\right].
    \label{app-eq:empirical_v_inf_kernel_conditional_moment_constraints_dual}
\end{equation}
as $\hat\eta_f$ and $\hat\eta_\KCMC$, we get
\begin{equation}
    \hat w_\KCMC(y, t, x) \in  \left( \frac{1}{p_\obs(t|x)} \right)
    \partial f_{t, x}^* \left(
        \frac{{{}\hat\eta_\KCMC}^T\bspsi(t, x) - r(y, t, x)}{{\hat\eta}_f}
    \right),
    \label{eq:empirical_w_solution_kernel_conditional_moment_constraints}
\end{equation}
which proves \eqref{eq:empirical_kcmc_solution_characterization}.

\subsection{Specification Error}

Using the above characterization of the solutions, we can find a condition under which the specification error of estimator $\hat V_\textinf^\KCMC(\pi)$ becomes zero for policy $\pi$ so that $\left| V_\textinf^\KCMC(\pi) - V_\textinf^\CMC(\pi) \right| = 0$.

\begin{theorem}[No specification error]\label{app-th:no_specification_error}
Let $\eta^*_\CMC(t, x)$ be the solution of dual problem \eqref{app-eq:v_inf_conditional_moment_constraints_dual} for $V_\textinf^\CMC(\pi)$. Then, if 
\begin{equation}
    \eta_\CMC^* \in \mathrm{span}\left(\{\psi_1, \ldots, \psi_D\}\right),
    \label{app-eq:eta_cmc_in_kernel_subspace}
\end{equation}
we have
\begin{equation}
    V_\textinf^\CMC(\pi) = V_\textinf^\KCMC(\pi).
\end{equation}
\end{theorem}

\begin{proof}
Take $\eta^*_f$ such that $\left(\eta^*_\CMC(t, x), \ \eta^*_f\right)$ is the solution of dual problem \eqref{app-eq:v_inf_conditional_moment_constraints_dual} for $V_\textinf^\CMC$.
Then, we can take multiplier $\eta_\KCMC^*$ that satisfies $\eta_\CMC^* = {\eta_\KCMC^*}^T \bspsi$.
Now, we can see that dual problem \eqref{app-eq:v_inf_kernel_conditional_moment_constraints_dual} for $V_\textinf^\KCMC$ can be considered as the restricted version of dual problem \eqref{app-eq:v_inf_conditional_moment_constraints_dual} for $V_\textinf^\CMC$
where $\eta_\CMC$ is constrained to the subspace spanned by $\{\psi_1, \ldots, \psi_D\}$.
Therefore, as restricted solution $\left(\eta^*_\KCMC, \ \eta^*_f\right)$ achieves the same value as the solution of the non-restricted problem, it is clearly a solution of restricted problem \eqref{app-eq:v_inf_kernel_conditional_moment_constraints_dual}.
Finally, owing to the strong duality, we can calculate the values of $V_\textinf^\CMC$ and $V_\textinf^\KCMC$ by the values of the dual problems, which implies $V_\textinf^\CMC = V_\textinf^\KCMC$.
\end{proof}

Interestingly, with the above result, we can derive quantile balancing constraint \eqref{eq:quantile_balancing_constraints} for the previously proposed sharp estimator by \citet{dorn2022sharp}.

\begin{example}[Derivation of quantile balancing estimator \citep{dorn2022sharp}]\label{ex:quantile_balancing_as_a_special_case}
Let us consider the same box constraints as Example \ref{ex:box_constraint_analytical_solution}.
For this problem, we know the analytical form of dual solution $\eta^*_\CMC(t, x) = \left(\frac{\pi(t|x)}{p_\obs(t| x)}\right)Q(t, x)$.
Therefore, we can take $D=1$ and set $\psi_1(t, y)= \left(\frac{\pi(t|x)}{p_\obs(t| x)}\right)Q(t, x)$ to meet condition \eqref{app-eq:eta_cmc_in_kernel_subspace} in Theorem \ref{app-th:no_specification_error} to obtain the kernel conditional moment constraint with no specification error.
\end{example}

\subsection{Consistency of Policy Evaluation and Learning}

Lastly, we study empirical estimator $\hat V_\textinf^\KCMC$ and provide convergence guarantees for policy evaluation and learning.
First, we prove the consistency of our estimator for fixed policy $\pi$ by reduction of our problem to the M-estimation \citep{van2000empirical} using the dual formulation.
We similarly show a reduction of policy learning problem $\max_{\pi\in\Pi}\hat V_\textinf^\KCMC$ to the M-estimation and prove the consistency of the learned policy parameter when policy class is finite-dimensional and concave.
This reduction to the M-estimation significantly simplified our proof compared to the proof in \citet{kallus2021minimax} using uniform convergence, because we can take advantage of the well-studied theory of M-estimation.
Though their approach using uniform convergence is very powerful, we found it not immediately applicable to our work, due to the difficulty of taking the empirical moment constraints into account.

\subsubsection{Consistency of Policy Evaluation}
To prove the consistency of policy evaluation and learning, we will make use of the two following convergence lemmas for loss function $\ell:\Theta\times\cZ\to\RR$, where $\Theta\subseteq\RR^K$ for some $K$ and $\cZ:=\cY\times\cT\times\cX$.
We assume that for $Z:=(Y, T, Z)$, the loss function satisfies $\EE|\ell_\theta(Z)| < \infty$ for any $\theta\in\Theta$.

\begin{lemma}[Uniform convergence on compact space, {\citet[Lemma 5.2.2.]{van2000empirical}}]\label{lemma:uniform_convergence_compact_space}
Assume that parameter space $\left(\Theta, \|\cdot\|\right)$ is compact.
Also assume that $\theta \mapsto \ell_\theta$, $\theta \in \Theta$ is continuous and it has an $L^1$ envelope so that $\EE[G(Z)]<\infty$ for $G(z):=\sup_{\theta\in\Theta} \left|\ell_\theta(z)\right|$.
Then, $\sup_{\theta\in\Theta}\left| \hat\EE_n[\ell_\theta(Z) - \EE[\ell_\theta(Z) \right| \pto 0$.
\end{lemma}

\begin{lemma}[Consistency of convex M-estimation, {\citet[Lemma 5.2.3.]{van2000empirical}}]\label{lemma:m_estimation}
Let us define $\theta_0 \in\arg\min_{_\theta \in \Theta}\EE[\ell_\theta(Z)]$ and its M-estimator $\hat\theta_n \in\arg\min_{\theta \in \Theta}\hat\EE_n[\ell_\theta(Z)]$.
Suppose that $\theta_0$ is the unique minimizer and that $\Theta\subseteq\RR^k$ is convex.
Also assume that $\theta \mapsto \ell_\theta$, $\theta \in \Theta$ is continuous and convex, satisfying $\EE[G_\varepsilon]<\infty$ for $G_\varepsilon(z):=\sup_{\theta\in\Theta:\ \|\theta - \theta_0\| \leq \varepsilon} \left|\ell_\theta(z)\right|$ for some $\varepsilon > 0$.
Then, $\hat \theta_n \pto \theta_0$.
\end{lemma}

As our dual problem for policy evaluation \eqref{app-eq:v_inf_kernel_conditional_moment_constraints_dual} and \eqref{app-eq:empirical_v_inf_kernel_conditional_moment_constraints_dual} are concave maximization, we can immediately apply the above lemma as follows.

\begin{theorem}[Consistency of policy evaluation]
Define parameter space of $\left(\eta_f, \eta_\KCMC\right)$ as $\Theta \subseteq \RR_+\times\RR^D$.
Further, define 
$\theta_0 :=(\eta^*_f, \eta^*_\KCMC)$ as the solution of dual problem \eqref{app-eq:v_inf_kernel_conditional_moment_constraints_dual} for $V_\textinf^\KCMC$ and 
$\hat\theta_n :=(\hat\eta_f, \hat\eta_\KCMC)$ as the solution to dual problem \eqref{app-eq:empirical_v_inf_kernel_conditional_moment_constraints_dual} for $\hat V_\textinf^\KCMC$.
Now, assume that $\theta_0$ is unique and that $\theta_0\in\Theta$.
Also assume that $f^*_{t, x}:\RR\to\RR$ is continuous for any $t\in\cT$ and $x\in\cX$.
Define $\ell:\Theta\times\cZ\to\RR$ as
\begin{equation}
    \ell_\theta (t, y, x) := \eta_f \gamma - {\eta_\KCMC}^T \bspsi(t, x)
    + \eta_f
        f_{t, x}^*  \left( \frac{{\eta_\KCMC}^T \bspsi(t, x) - r(y, t, x)}{\eta_f} \right)
\end{equation}
so that it is the negative version of the inside of the expectation of dual objectives \eqref{app-eq:v_inf_kernel_conditional_moment_constraints_dual} and  \eqref{app-eq:empirical_v_inf_kernel_conditional_moment_constraints_dual}.
Furthermore, assume $\ell$ satisfies $\EE\left|\ell_\theta\right|<\infty$ for any $\theta\in\Theta$ and $\EE[G_\varepsilon]<\infty$ for $G_\varepsilon:=\sup_{\theta\in\Theta:\ \|\theta - \theta_0\| \leq \varepsilon} \left| \ell_\theta \right|$ for some $\varepsilon > 0$.
Then, we have 
$\hat\theta_n\pto\theta_0$
and $\hat V_\textinf^\KCMC  \pto V_\textinf^\KCMC$.
\end{theorem}
\begin{proof}
We can immediately apply Lemma \ref{lemma:m_estimation} and get $\hat\theta_n\pto\theta_0$.
Thus, $\hat\theta_n$ tend to the inside of compact set $\{\theta\in\Theta:\ \|\theta - \theta_0\| \leq \varepsilon\}$, in which we have the uniform convergence of $\hat\EE_n[\ell_\theta(Z)]$ to $\EE[\ell_\theta(Z)]$.
Therefore, we have $\hat V_\textinf^\KCMC = -\hat\EE_n[\ell_{\hat\theta_n}(Z)] \pto -\EE[\ell_{\hat\theta_n}(Z)] \pto  - \EE[\ell_{\theta_0}(Z)] = V_\textinf^\KCMC$.
\end{proof}

In practice, it is difficult to check the assumption of integrability condition $\EE|\ell_\theta| < \infty$ for any $\theta\in\Theta$ as well as the uniqueness of the solution.
However, it is possible in some cases to check $L^1$ envelope condition $\EE[G_\varepsilon] < \infty$, because local Lipschitzness of $f^*$ implies the existence of such $\varepsilon$.
For example, for the box constraints of Example \ref{ex:box_constraint_analytical_solution}, we know that $f^*$ is upper bounded by $b_{\tilde w}(t, x)$.
For the f-divergence constraints, the conjugate function $f^*$ for many choices of f-divergence is locally Lipschitz, as shown in Table \ref{table:f_conjugate_list}.


\subsubsection{Consistency of Concave Policy Learning by M-estimation}

Now we consider policy learning.
Instead of providing the standard uniform convergence proof, our theoretical result leverages the preceding lemmas.

\begin{theorem}[Consistency of concave policy learning]
Assume concave policy class $\{\pi_\beta(t|x):\ \beta\in\mathcal{B}\}$ with convex parameter space $\mathcal{B}$ satisfying that $\beta\mapsto\pi_\beta(t|x)y$ is concave for any $y\in\cY$, $t\in\cT$ and $x\in\cX$.
Define the parameter space of $\theta=\left(\beta, \eta_f, \eta_\KCMC\right)$ as $\Theta = \mathcal{B}\times\Theta_\eta$ for some convex $\Theta_\eta\subseteq\RR_+\times\RR^D$.
Define also
$\theta_0 :=(\beta^*, \eta^*_f, \eta^*_\KCMC)$ 
so that $\beta^* \in \arg\max_{\beta\in\mathcal{B}} V_\textinf^\KCMC(\pi_\beta)$ 
and $\left(\eta^*_f, \eta^*_\KCMC\right)$ is the solution of dual problem \eqref{app-eq:v_inf_kernel_conditional_moment_constraints_dual}
for $V_\textinf^\KCMC$ at policy $\pi_{\beta^*}$. 
Similarly, define
$\hat\theta_n :=(\hat\beta, \hat\eta_f, \hat\eta_\KCMC)$ 
so that $\hat\beta \in \arg\max_{\beta\in\mathcal{B}} \hat V_\textinf^\KCMC(\pi_\beta) $ 
and $\left(\hat\eta_f, \hat\eta_\KCMC\right)$ is the solution of dual problem 
\eqref{app-eq:empirical_v_inf_kernel_conditional_moment_constraints_dual}
for $\hat V_\textinf^\KCMC$ at policy $\pi_{\hat\beta}$. 
Now, assume that $\theta_0$ is unique and that $\theta_0\in\Theta$.
Also assume that $f^*_{t, x}:\RR\to\RR$ and $\beta\to\pi_\beta$ is continuous for any $t\in\cT$ and $x\in\cX$. 
Define $\ell:\Theta\times\cZ\to\RR$ as
\begin{equation}
    \ell_\theta (t, y, x) := \eta_f \gamma - {\eta_\KCMC}^T \bspsi(t, x)
    + \eta_f
        f_{t, x}^*  \left( \frac{{\eta_\KCMC}^T \bspsi(t, x) - \left(\frac{\pi_\beta(t|x)}{p_\obs(t|x)}\right)y}{\eta_f} \right)
\end{equation}
so that it is the negative version of the inside of the expectation of dual objectives \eqref{app-eq:v_inf_kernel_conditional_moment_constraints_dual} and  \eqref{app-eq:empirical_v_inf_kernel_conditional_moment_constraints_dual}.
Furthermore, assume $\ell$ satisfies $\EE\left|\ell_\theta\right|<\infty$ for any $\theta\in\Theta$ and $\EE[G_\varepsilon]<\infty$ for $G_\varepsilon:=\sup_{\theta\in\Theta:\ \|\theta - \theta_0\| \leq \varepsilon} \left| \ell_\theta \right|$ for some $\varepsilon > 0$.
Then, we have
$\hat\theta_n\pto\theta_0$
and 
$\hat V_\textinf^\KCMC(\pi_{\hat\beta}) \pto V_\textinf^\KCMC(\pi_{\beta^*})$.
\end{theorem}
\begin{proof}
Due to the concavity of policy class $\{\pi_\beta(t|x):\ \beta\in\mathcal{B}\}$, 
we know that $\beta\mapsto\EE\left[\tilde w(Y, T, X)\left(\frac{\pi_\beta(T|X)}{p_\obs(T|X)}\right)Y\right]$ is concave for any $t\in\cT$ and $x\in\cX$.
Then, we can see that 
\begin{align}
    \max_{\beta\in\mathcal{B}} V_\textinf^\KCMC 
    &= \max_{\beta\in\mathcal{B}} \max_{\substack{\eta_\KCMC\in\RR^D\\\eta_f>0}} \min_{\tilde w}
    \EE\left[\tilde w(Y, T, X) \left( \frac{\pi_\beta(T|X)}{p_\obs(T|X} \right) Y \right]  \\
    &\quad\quad\quad\quad\quad\quad\quad\quad\quad\quad
    - \EE\left[(\tilde w - 1) {\eta_\KCMC}^T\bspsi\right] 
    + \eta_f\left(\EE[f_{T, X}(\tilde w)] - \gamma\right) \\
    &= \max_{\beta\in\mathcal{B}}
        \max_{\substack{\eta_\KCMC\in\RR^D\\\eta_f>0}}
        \EE[ - \ell_\theta(Y, T, X)] \\
    &= \max_{\theta\in\Theta} \EE[ - \ell_\theta(Y, T, X)] 
\end{align}
is a concave maximization problem, because $\EE[\ell_\theta]$ is the pointwise infimum of concave functions.
Thus, we can apply Lemma \ref{lemma:m_estimation} and get $\hat\theta_n\overset{p.}{\to}\theta_0$,
which implies $\hat\theta_n$ tend to the inside of compact set $\{\theta\in\Theta:\ \|\theta - \theta_0\| \leq \varepsilon\}$, where we have uniform convergence guarantee of $\hat\EE_n[\ell_\theta(Z)]$.
Therefore, we have $\hat V_\textinf^\KCMC(\pi_{\hat \beta}) = -\hat\EE_n[\ell_{\hat\theta_n}(Z)] \pto -\EE[\ell_{\hat\theta_n}(Z)] \pto  - \EE[\ell_{\theta_0}(Z)] = V_\textinf^\KCMC(\pi_{\beta ^*})$.
\end{proof}

\section{Experimental Settings and Additional Numerical Examples}
Lastly, we provide the details of our numerical experiments and provide more experimental results of the f-sensitivity models.

\subsection{Datasets}
In the experiment, we used two types of data, one is synthetic and the other is real-world data.
We base most of the experiments on the first synthetic data adopted from \citet{kallus2018confounding, kallus2021minimax}.
The second dataset is a real-world example used in \citet{dorn2022sharp}, and it is used to illustrate the application of our methods to a real-world dataset.

Now we explain our first dataset. 
We use the following data-generating process for this synthetic data.
\begin{align}
    \xi &\sim \mathrm{Bern}(1/2), \\
    X &\sim \mathcal{N}(\mu_x, \mathrm{I}_5), \\
    Y_0|X, \xi &\sim \mathcal{N}(\beta_{x, 0}^T X + \beta_{\xi, 0} \xi + \beta_{\text{const}, 0}, 1), \\
    Y_1|X, \xi &\sim \mathcal{N}(\beta_{x, 1}^T X + \beta_{\xi, 1} \xi + \beta_{\text{const}, 1}, 1), \\
    U &= \bbmone_{Y_0 > Y_1}, \\
    T|X, Y_1, Y_2, U &\sim \mathrm{Bern}(e(X, U)), \\
    Y &= Y_T,
    \label{eq:experimental_data_1}
\end{align}
where $e(X, U):= \frac{6e(X)}{4 + 5U + e(X)(2 - 5U)}$ and $e(x):= \sigma(\beta_e^T x)$. Here, $\sigma(u):=\frac{\exp(u)}{1 + \exp(u)}$ indicates the sigmoid function.
The parameters we used are
\begin{align*}
    \mu_x &= [-1, 0.5, -1, 0, -1], \\
    \beta_{x, 0} &= [0, .5, -0.5, 0, 0], \\
    \beta_{x, 1} &= [-1.5, 1.5, -2, 1, 0.5], \\
    \beta_{\xi, 0} &= 1, \\
    \beta_{\xi, 1} &= -1, \\
    \beta_{\text{const}, 0} &= 2.5,\\
    \beta_{\text{const}, 1} &= -0.5,\\
    \beta_e &= [0, 0.75, -0.5, 0, -1].
\end{align*}
For the policy evaluation task, we used policy $\pi(t=1|x):=e(x)$.

For the real-world data example, we use the same dataset as \citet{dorn2022sharp}, which is 668 subsamples of data from the 1966-1981 National Longitudinal Survey (NLS) of Older and Young Men. 
The subsamples consist of the 1978 cross-section of Young Men who are craftsmen or laborers and are not enrolled in school. 
We estimate the average treatment effect ($\EE[Y|T=1] - \EE[Y|T=0]$) of union membership ($T$) on log wages ($Y$), and eight other covariates are used as $X$. 
For the average treatment effect estimation, we substitute $\bbmone_{t=1} - \bbmone_{t=0}$ in place of $\pi(t|x)$.
Note that this substituted quantity is a difference of policy $\pi_1(t|x):= \bbmone_{t=1}$ and $\pi_0(t|x):= \bbmone_{t=0}$, and therefore it is not a proper policy. 
Nevertheless, such a substitution is still possible, as our method can accommodate any function $\pi(t|x)$ in place of the evaluated policy.
\footnote{If we re-define the reward and the evaluated policy as $Y':= 2(\bbmone_{T=1}Y - \bbmone_{T=0}Y)$ and $\pi(t|x):=\frac{1}{2}$, the offline policy evaluation of such a setup is equivalent to the average treatment effect estimation.}

For synthetic data, we generate a dataset of 500 samples for individual experimental configurations, unless otherwise specified.
We repeat the experiment 10 times using different random seeds and report the mean of the 10 experiments.
Additionally, we indicate plus/minus one standard deviation from the mean by the colored band around the line representing the mean value.

Lastly, conditional probability $p_\obs(t|x)$ used to construct the estimators was estimated from the data using the logistic regression with covariate $X$.

\subsection{Compared Estimators}
In the numerical examples, we consider four types of estimators.

First, as the baseline method, we consider the conventionally used ZSB estimator.
We impose the ZSB constraints on other estimators, in order to see the additional improvements by these constraints.
To impose ZSB constraints while ensuring the feasibility of the associated convex programming, we applied appropriate rescaling to the estimates of $p_\obs(t|x)$.
\footnote{We multiplied estimate $\hat p_\obs(t|X_i)$ by $\frac{1}{n}\sum_{i=1}^n[ \bbmone_{T_i=t} / \hat p_\obs(T_i|X_i)]$. When all the constraints can be expressed as linear constraints, fractional linear programming can be used to impose the ZSB constraints in a more natural manner \citep{zhao2019sensitivity}. However, the quadratic constraints of GP KCMC make it impossible to use the linear fractional programming approach.}
As discussed above, this estimator is conventionally used \citep{tan2006distributional, kallus2018confounding, zhao2019sensitivity} but is known to provide too conservative bounds.

Against this baseline, we compared two types of the proposed estimators based on the kernel moment constraints (KCMC), which are the low-rank Gaussian process KCMC and the low-rank hard KCMC using the orthogonal function class obtained by the kernel PCA.
In the following, we call them "low-rank GP KCMC" and "low-rank hard KCMC", respectively.
For both of the low-rank KCMC estimators, we used $100$-rank approximation so that $D=100$, unless otherwise specified.

Additionally, we compared the quantile balancing (QB) estimator by \citet{dorn2022sharp}.
As discussed in Example \ref{ex:quantile_balancing_as_a_special_case}, this estimator can be considered as a special case of the low-rank hard KCMC estimators that uses the optimal orthogonal function class, in the case of box constraints.

To solve the convex programming involved in the above estimators, we used MOSEK \citep{aps2019mosek} and ECOS \citep{bib:Domahidi2013ecos} through the API of CVXPY \citep{diamond2016cvxpy}.

\subsection{Additional Numerical Experiments with f-divergence Sensitivity Models}

Here, we list more examples of the f-sensitivity analysis with various types of f-divergences using the synthetic dataset.
The f-divergences considered here are listed in Table \ref{table:f_conjugate_list}.

\begin{table}[htbp]
    \centering
    \caption{Commonly used f-divergence, corresponding convex function $f:\RR_+\to\RR$, and its Fenchel conjugate $f^*:\mathrm{dom}(f^*)\to\RR$.}
    \begin{tabular}{ l || c | c | c }
      \hline
       f-divergence & $f(u)$ & $f^*(v)$ & $\mathrm{dom}(f^*)$\\ 
      \hline
      \hline
      KL & $u\log u$ & $\exp(v - 1)$ & $\RR$\\ 
      \hline
      Reverse KL & $-\log u$ & $-1 - \log (-v)$ & $\RR_-$ \\ 
      \hline
      Jensen-Shannon & \st{$-(u+1)\log\left(\frac{u + 1}{2}\right) + u \log  u$} (See Appendix \ref{app:errata}) & $-\log(2 - \exp(v))$ & $v < \log 2$\\ 
      \hline
      Squared Hellinger & $(\sqrt{u}-1)^2$ & $\frac{v}{1 - v} $ & $v<1$\\ 
      \hline
      Pearson $\chi^2$ & $(u-1)^2$ & $\frac{1}{4}v^2 + v$ & $\RR$\\ 
      \hline
      Neyman $\chi^2$ & $\frac{1}{u}-1$ & $-2 \sqrt{-v} + 1$ & $\RR_-$\\ 
      \hline
      Total Variation & $\frac{1}{2}|u - 1|$ & $v$ & $-\frac{1}{2} \leq v \leq  \frac{1}{2}$ \\ 
      \hline
    \end{tabular}
    \label{table:f_conjugate_list}
\end{table}

Similarly to the case of Tan's marginal sensitivity models, the sharp estimators are tighter than the ZSB estimator.
The low-rank hard KCMC is also providing (potentially excessively) tighter bounds than the low-rank Gaussian process KCMC and the quantile balancing estimators.

Interestingly, the quantile balancing constraint \footnote{We used $\Gamma = 1.5$, which corresponds to finding $\hat Y(t, x)$ that approximates the $40$ percentile.} is providing almost as sharp bound as the GP KCMC-based methods, even though it is no longer the theoretically optimal constraint.
Still, it is possible to give some intuition on the use of orthogonal function class $\{\phi_1\}$ with $\phi_1(t, x) := \left(\frac{\pi(t|x)}{p_\obs(t|x}\right)\hat Y(t, x)$ as follows:
If we have a good regressor of $Y$ satisfying $\hat Y(T, X) \approx Y$ and if $\tilde w$ follows constraint 
$\EE_\obs\left[(\tilde w(Y, T, X) - 1) \left(\frac{\pi(T|X)}{p_\obs(T|X)}\right) \hat Y(T, X)\right] = 0$,
we have
\begin{align}
\EE_\obs& \left[\tilde w(Y, T, X) \left(\frac{\pi(T|X)}{p_\obs(T|X)}\right) Y(T, X)\right] \\
&= \EE_\obs\left[\tilde w(Y, T, X) \left(\frac{\pi(T|X)}{p_\obs(T|X)}\right) \{(Y - \hat Y(T, X)) + \hat Y(T, X)\}\right] \\
&= \EE_\obs\left[\tilde w(Y, T, X) \left(\frac{\pi(T|X)}{p_\obs(T|X)}\right) ((Y - \hat Y(T, X))\right] 
+ \EE_\obs\left[\left(\frac{\pi(T|X)}{p_\obs(T|X)}\right) \hat Y(T, X)\right] \\
&\approx \EE_\obs\left[\left(\frac{\pi(T|X)}{p_\obs(T|X)}\right) \hat Y(T, X)\right].
\end{align}
This implies that if $Y$ is easy to predict with $X$ and $T$, the orthogonal constraint using the regressor $\hat Y(t, x)$ gives a similar value to the IPW estimator of $\hat Y(T, X)$ as long as $\tilde w$ is not too far from $1$.

\begin{figure}[htbp]
     \centering
     \begin{subfigure}[b]{0.45\textwidth}
     \centering
        \includegraphics[width=\linewidth]{figures/policy_evaluation_synthetic_binary_changing_gamma_KL_}
         \caption{KL sensitivity model}
     \end{subfigure}
     \begin{subfigure}[b]{0.45\textwidth}
     \centering
        \includegraphics[width=\linewidth]{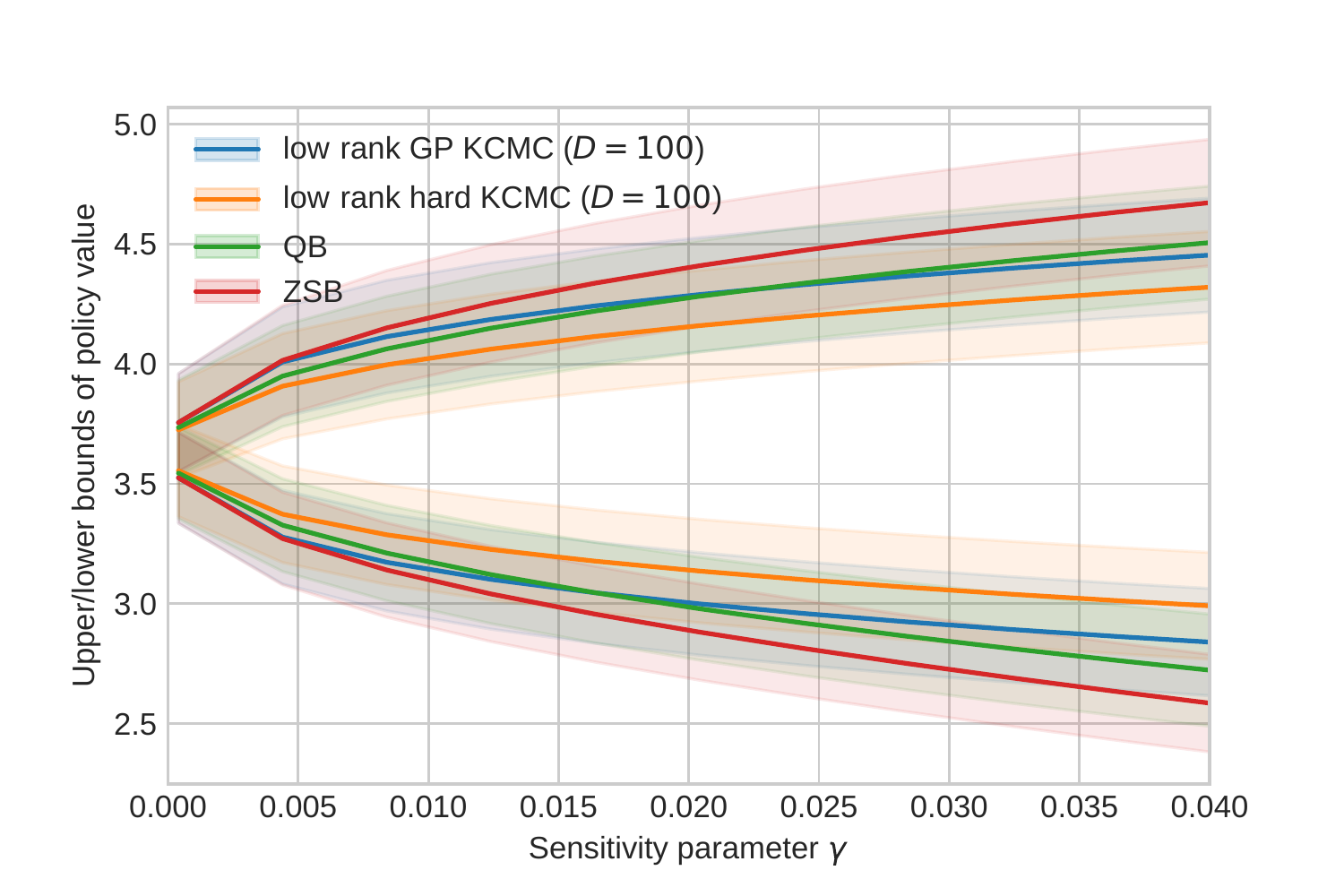}
         \caption{Reverse KL sensitivity model}
     \end{subfigure}
     \begin{subfigure}[b]{0.45\textwidth}
     \centering
        \includegraphics[width=\linewidth]{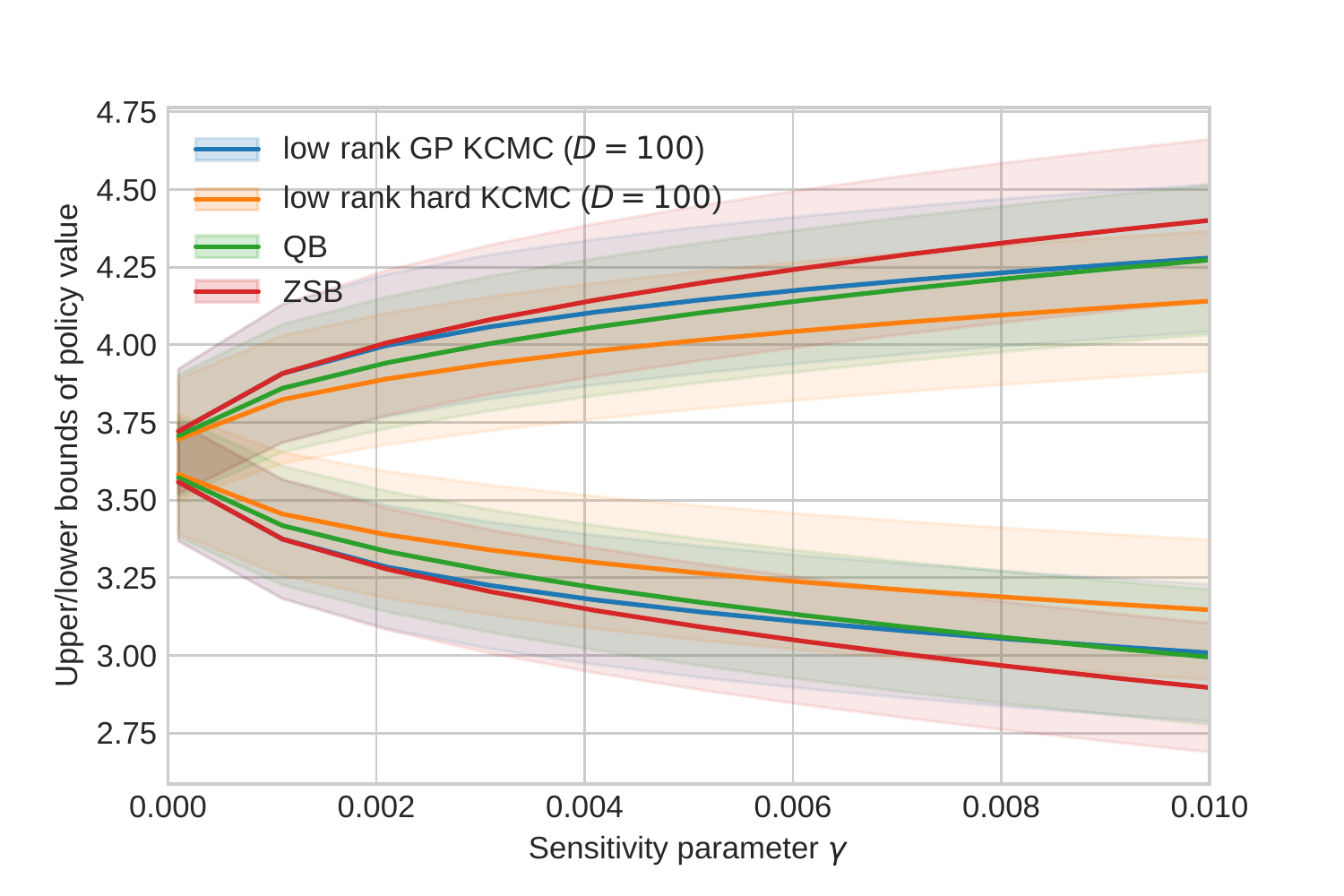}
         \caption{Squared Hellinger sensitivity model}
     \end{subfigure}
     \begin{subfigure}[b]{0.45\textwidth}
     \centering
        \includegraphics[width=\linewidth]{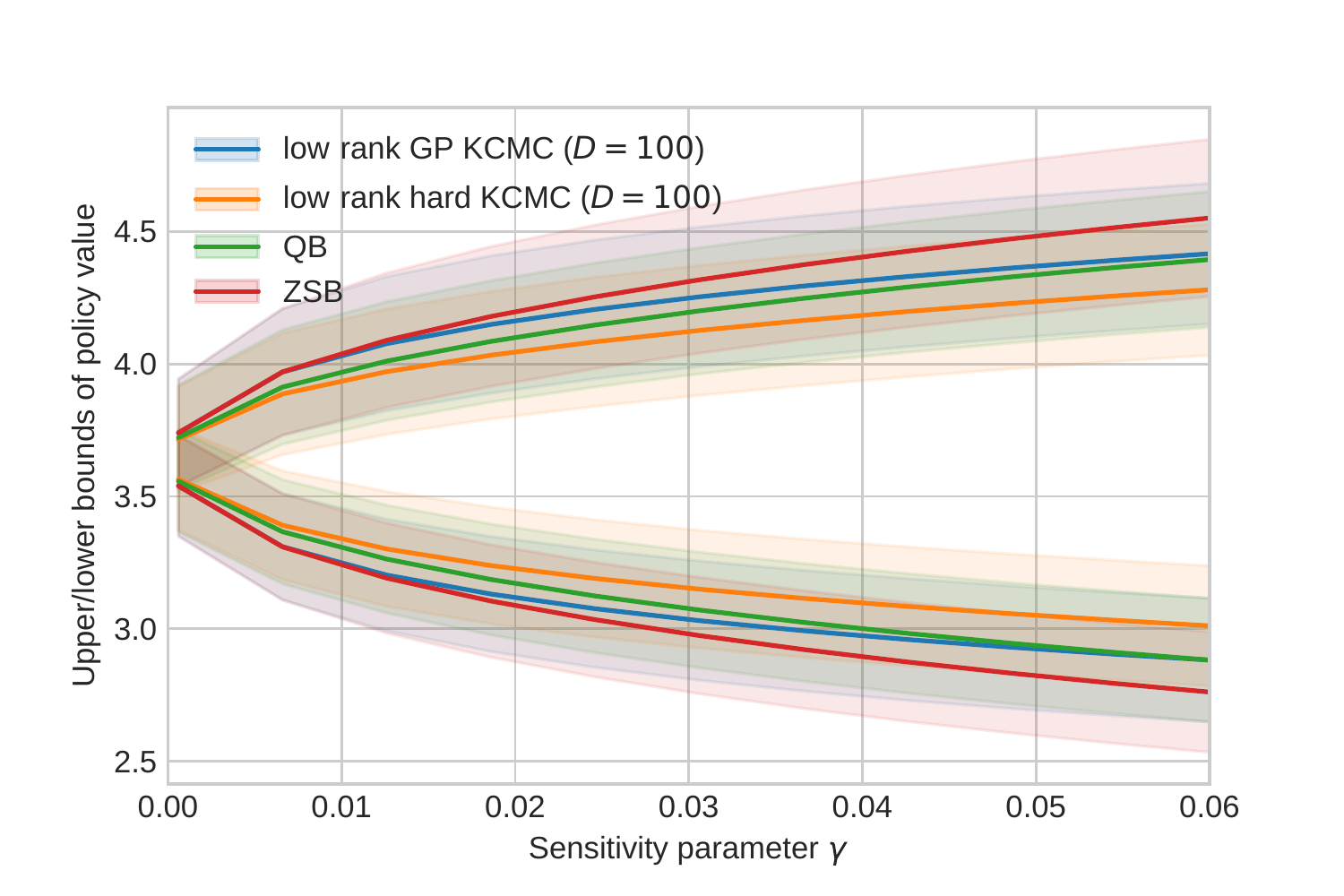}
         \caption{Pearson $\chi^2$ sensitivity model}
     \end{subfigure}
     \begin{subfigure}[b]{0.45\textwidth}
     \centering
        \includegraphics[width=\linewidth]{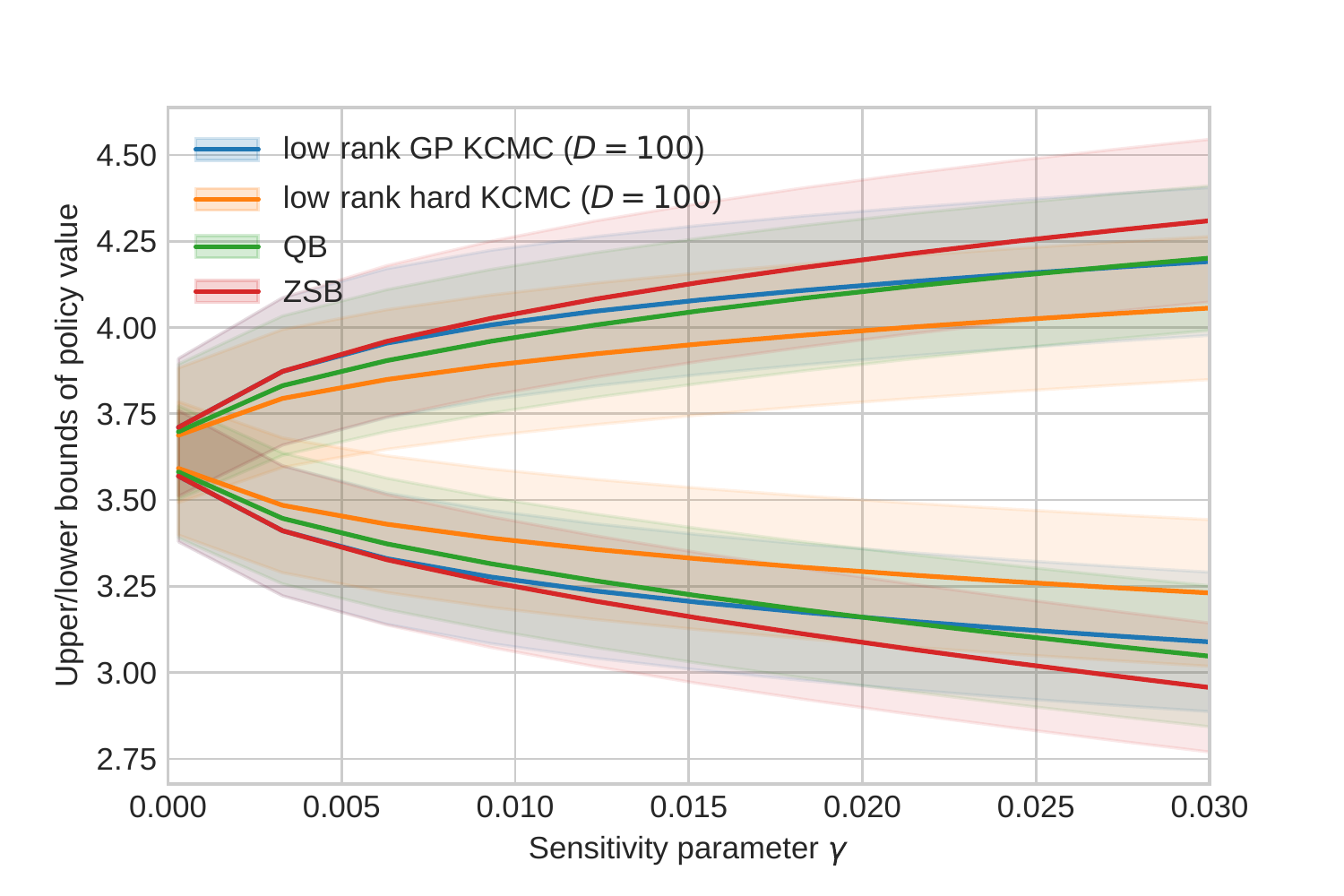}
         \caption{Neyman $\chi^2$ sensitivity model}
     \end{subfigure}
     \begin{subfigure}[b]{0.45\textwidth}
     \centering
        \includegraphics[width=\linewidth]{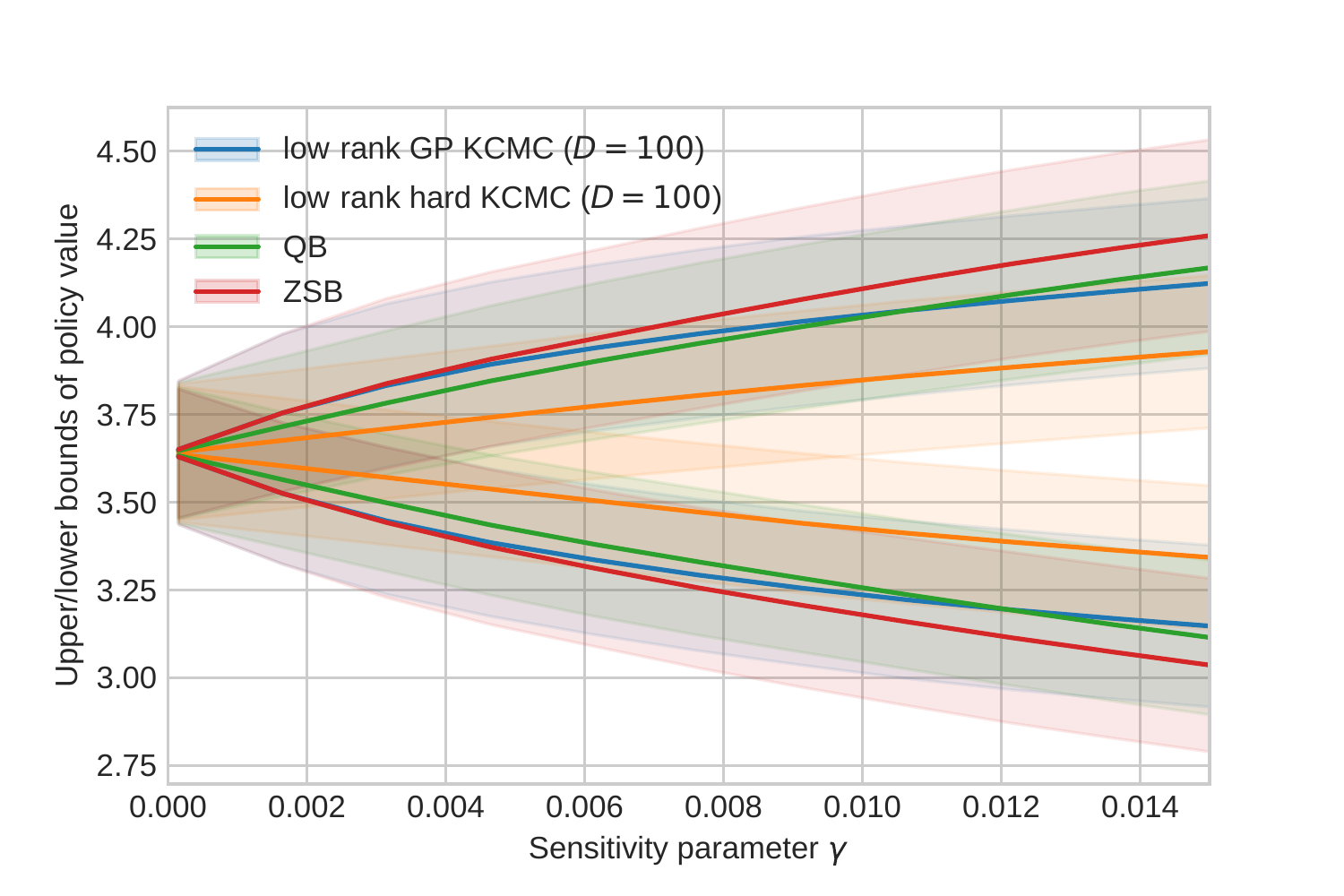}
         \caption{Total variation sensitivity model}
     \end{subfigure}
    \caption{Estimated upper and lower bounds of policy value using different estimators for different f-sensitivity models. The synthetic data is used.}
\end{figure}

\newpage

\section{A List of Errata}\label{app:errata}

\begin{itemize}
    \item Function $f$ for Jensen-Shannon divergence in Table \ref{table:f_conjugate_list} must be divided by $2$ so that $f(u) = -\frac{1}{2}(u+1)\log\left(\frac{u + 1}{2}\right) + \frac{1}{2}u \log u$.
    \item In many parts of this paper, $f^*(v)$ must be replaced by $f^{*_\nearrow}(v):=\inf_{v \leq \tilde v} f^*(\tilde v)$, which is an infimal convolution of $f^*$ and $(I_{[0, \infty)})^*(v) = I_{(-\infty, 0]}(v)$, which is known to be convex. This modification is required because we forgot to include multipliers for constraints $0 \leq \tilde w$ in Equations \eqref{app-eq:v_inf_conditional_moment_constraints_dual} and \eqref{app-eq:v_inf_kernel_conditional_moment_constraints_dual} when deriving the dual problem.
    Alternatively, we can assume the conditional f-constraint already includes the condition $w\geq 0$ so that $f_{x, t}(v) = \infty$ for any $v<0$.
    \item In derivation of dual problem \eqref{app-eq:empirical_v_inf_kernel_conditional_moment_constraints_dual}, \eqref{app-eq:v_inf_kernel_conditional_moment_constraints_dual}, and \eqref{app-eq:empirical_v_inf_kernel_conditional_moment_constraints_dual}, we can assume that $\eta_f>0$ except for the case of box constraints.
    In the case of box constraints, the conditional f-constraint is not tight, implying $\eta_f = 0$.
    In such a case, the last line of \eqref{app-eq:empirical_v_inf_kernel_conditional_moment_constraints_dual}, for example, simplifies to $\EE\left[ \eta_\CMC(T, X) - f^*_{T, X} \left(\eta_\CMC(T, X) - r(Y, T, X)\right)\right]$.
\end{itemize}

\end{document}